\documentclass[11pt]{article}

\usepackage[final]{acl}

\usepackage{times}
\usepackage{latexsym}

\usepackage[T1]{fontenc}

\usepackage[utf8]{inputenc}

\usepackage{microtype}

\usepackage{inconsolata}

\usepackage{graphicx}
\graphicspath{{figures/}}

\usepackage{amsmath}
\usepackage{amssymb}
\usepackage{amsthm}
\usepackage{booktabs}
\usepackage{multirow}
\usepackage{algorithm}
\usepackage{algpseudocode}
\usepackage{xcolor}
\usepackage{subcaption}
\usepackage{enumitem}
\usepackage{xurl}

\usepackage{siunitx}
\newcommand{\apptabwide}{\scriptsize\setlength{\tabcolsep}{3pt}}

\newtheorem{theorem}{Theorem}
\newtheorem{proposition}{Proposition}
\newtheorem{corollary}{Corollary}
\newtheorem{definition}{Definition}

\newcommand{\calY}{\mathcal{Y}}
\newcommand{\calX}{\mathcal{X}}
\newcommand{\calC}{\mathcal{C}}
\newcommand{\calD}{\mathcal{D}}
\newcommand{\qhat}{\hat{q}}
\newcommand{\snc}{s_{\mathrm{nc}}}
\newcommand{\psocial}{P_{\mathrm{social}}}

\title{From Debate to Decision: Conformal Social Choice for Safe Multi-Agent Deliberation}

\author{
 \textbf{Mengdie Flora Wang\textsuperscript{1}},
 \textbf{Haochen Xie\textsuperscript{1}},
 \textbf{Guanghui Wang\textsuperscript{1}},
 \textbf{Aijing Gao\textsuperscript{2}},
\\
\textbf{Guang Yang\textsuperscript{1}},
 \textbf{Ziyuan Li\textsuperscript{3}},
 \textbf{Qucy Wei Qiu\textsuperscript{3}},
 \textbf{Fangwei Han\textsuperscript{3}},
\\
 \textbf{Hengzhi Qiu\textsuperscript{3}},
 \textbf{Yajing Huang\textsuperscript{3}},
 \textbf{Bing Zhu\textsuperscript{3}},
 \textbf{Jae Oh Woo\textsuperscript{1}}
\\
 \textsuperscript{1}AWS Generative AI Innovation Center, \\
 \textsuperscript{2}AWS Professional Services, \\
 \textsuperscript{3}HSBC Holdings Plc., HSBC Technology Center, China
}

\begin{document}
\maketitle

\begin{abstract}
Multi-agent debate improves LLM reasoning, yet agreement is not evidence of correctness.
When agents converge on a wrong answer through social reinforcement, consensus stopping commits that error to automated action with no recourse.
We introduce Conformal Social Choice, a post-hoc layer that converts debate outputs into calibrated act-versus-escalate decisions.
Verbalized probability distributions from heterogeneous agents are aggregated via a linear opinion pool and calibrated with split conformal prediction, yielding fixed-round marginal coverage: at any pre-specified round, the correct answer is included with probability ${\geq}\,1{-}\alpha$, without assumptions on individual model calibration.
A hierarchical action policy maps singletons to autonomous action and larger sets to human escalation.
Across eight MMLU-Pro domains and two additional benchmarks, fixed-round coverage stays close to target.
The key finding is not that debate becomes more accurate, but that the conformal layer makes its failures actionable: at the final round, 81.9\% of wrong-consensus cases are flagged by non-singleton sets at $\alpha{=}0.05$.
Separately, among cases that reach a singleton before the final round, accuracy reaches 90.0--96.8\% on MMLU-Pro (up to 22.1pp above consensus stopping): a selection effect, not a reasoning improvement.
Simple probability- and entropy-threshold rules attain comparable interception; conformal's distinct value is a finite-sample fixed-round coverage guarantee together with explicit candidate sets for escalation.
Adaptive first-singleton stopping is evaluated empirically and not covered by this guarantee.\footnote{\raggedright Code is available at \url{https://github.com/amazon-science/conformal-multi-agent-deliberation}.\par}
\end{abstract}

\begin{figure}[t]
  \centering
  \includegraphics[width=\columnwidth]{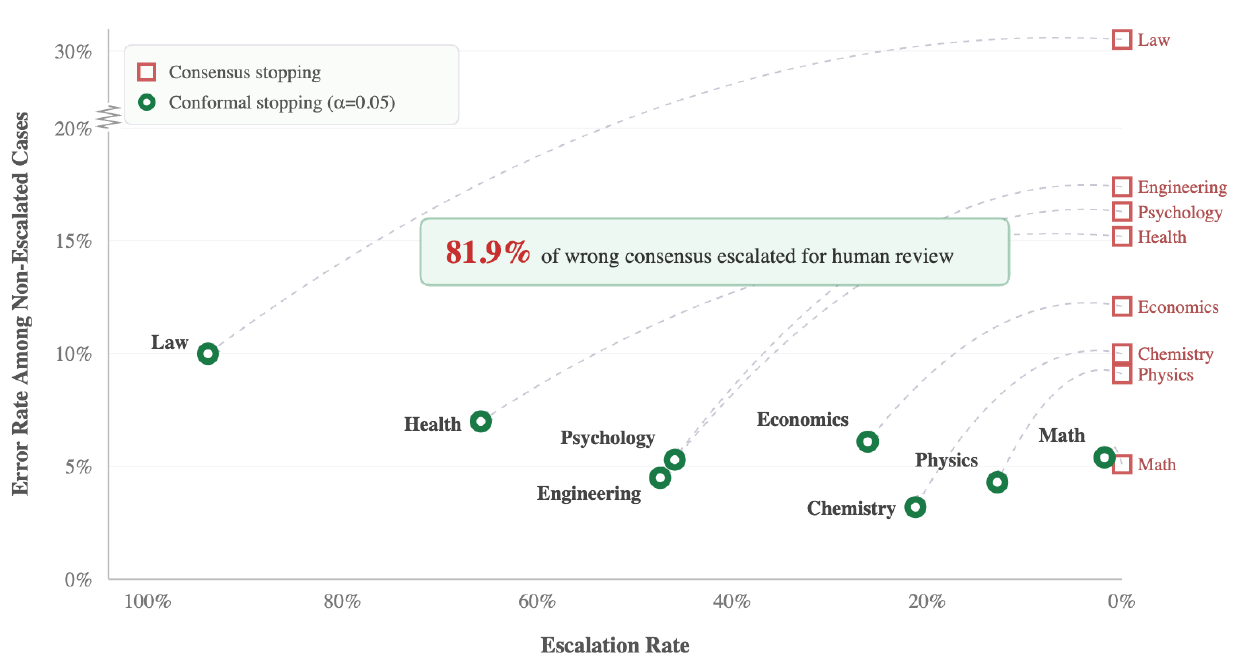}
  \caption{The cost of acting without calibrated refusal. Consensus stopping never escalates (0\% escalation rate) but commits to high error rates (up to 32.1\%). Conformal stopping trades automation for safety: by escalating uncertain cases to human review, it dramatically reduces error among the cases it does act on. At the final round with $\alpha{=}0.05$, 81.9\% of wrong-consensus cases have non-singleton conformal sets. Each arrow connects the same domain under the two stopping rules.}
  \label{fig:headline}
\end{figure}

\begin{figure*}[t]
  \centering
  \includegraphics[width=\textwidth]{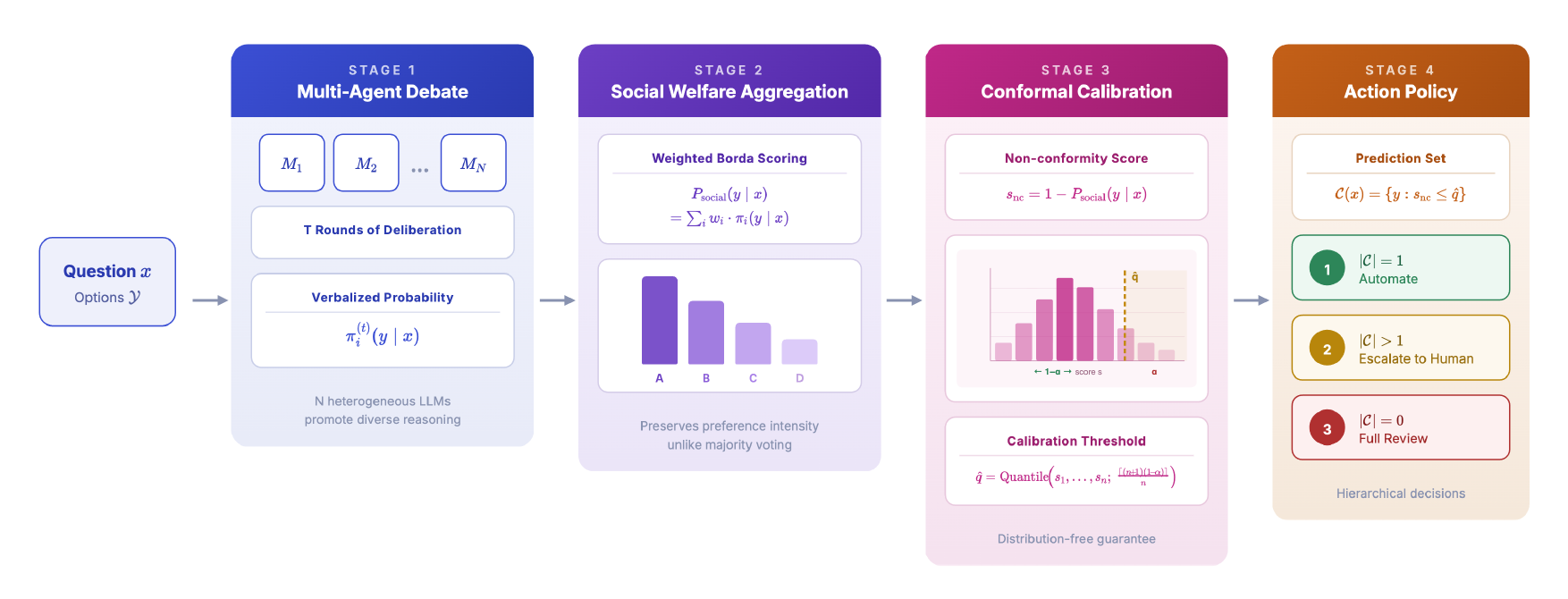}
  \caption{From debate output to deployment decision. Heterogeneous LLM agents debate for $T$ rounds, producing verbalized probability distributions (Stage~1). These are aggregated via a linear opinion pool into social probabilities (Stage~2), calibrated using a held-out set to compute the conformal threshold $\qhat_t$ (Stage~3), and used to construct prediction sets with a hierarchical action policy (Stage~4). At each pre-specified round, $\calC_t(x)$ satisfies marginal coverage $\Pr[y \in \calC_t(x)] \geq 1{-}\alpha$; this guarantee does not automatically extend to adaptive stopping.}
  \label{fig:pipeline}
\end{figure*}

\section{Introduction}
\label{sec:intro}

Multi-agent debate systems often end in unanimous answers, and unanimous answers can still be wrong.
Debate improves LLM reasoning \citep{du2023improvingfactualityreasoninglanguage, liang2023encouragingdivergentthinking}, and recent work refines \emph{when} agents converge \citep{hu2025multiagentdebatellmjudges} and \emph{when} debate outperforms voting \citep{choi2025debatevoteyieldsbetter}.
Many debate pipelines ultimately reduce the debate to a single point estimate (majority vote or argmax) and use agreement as a stopping signal.
\textbf{The real deployment question is not \emph{who won the debate?}\ but \emph{when is it safe to act?}}

Existing stopping rules do not directly address this question.
LLMs can conform to externally expressed preferences or opinions \citep{perez2023discovering, sharma2024towards}, and in multi-agent debate this social reinforcement can produce \emph{wrong consensus}: agents converge on an incorrect answer with high apparent confidence.
Stability-based stopping \citep{hu2025multiagentdebatellmjudges} detects agreement but cannot tell whether it is correct, and point-estimate aggregation (majority voting, weighted averaging, or argmax) then commits the system, discarding uncertainty that could have flagged the error.
The result is \emph{uncalibrated commitment}: wrong consensus feeds to automated action with no safety check.

The missing piece is not better debate but \emph{calibrated refusal}: the system should ask \emph{``is any answer confident enough to act on, or should a human decide?''}

We introduce \textbf{Conformal Social Choice}, a post-hoc decision layer that converts debate outputs into calibrated act-versus-escalate decisions, requiring no retraining and no access to model internals.
The pipeline aggregates verbalized probability distributions from heterogeneous LLM agents into a \emph{collective belief} via a linear opinion pool, then applies split conformal prediction to produce prediction sets satisfying a \emph{marginal coverage guarantee}: $\Pr[y \in \calC(x)] \geq 1{-}\alpha$, without assumptions on individual model calibration.
A hierarchical action policy maps singleton sets to autonomous action and larger sets to human escalation.
Empirically, the gains are substantial: at the final round with $\alpha{=}0.05$, 81.9\% of samples that end in wrong consensus have non-singleton conformal sets.
Separately, among cases resolved before the final round, conformal singletons are up to 22.1 percentage points more accurate than consensus stopping: a selection effect of calibrated refusal, not a reasoning improvement.

Our contributions are:
\textbf{(i)~Decision reframing:} We recast multi-agent debate from answer selection to set-valued action selection with fixed-round marginal coverage control, introducing a black-box, post-hoc pipeline that converts debate outputs into act-versus-escalate decisions (\S\ref{sec:method}).
\textbf{(ii)~Failure-mode quantification:} We show that 23.9\% of initially disputed cases end in unanimous wrong agreement; among all 586 final-round wrong-consensus cases, the round-3 conformal set is non-singleton on 81.9\% at $\alpha{=}0.05$ (\S\ref{sec:consensus_failure}--\ref{sec:conformal_safety}).
\textbf{(iii)~Operational mitigation and robustness:} Across MMLU-Pro, LogiQA, GPQA, a second agent ensemble, and repeated calibration splits, we characterize when calibrated candidate sets support safer escalation and when simple scalar abstention rules perform comparably (\S\ref{sec:results} and Appendix).
\section{Related Work}
\label{sec:related}

\paragraph{Multi-agent debate, voting, and opinion dynamics.}
Multi-agent debate improves factuality and reasoning \citep{du2023improvingfactualityreasoninglanguage, liang2023encouragingdivergentthinking}; subsequent work refined stopping criteria \citep{hu2025multiagentdebatellmjudges}, compared debate with voting \citep{choi2025debatevoteyieldsbetter}, modeled debate via Bayesian Nash Equilibrium \citep{yi2025debateequilibriumbeliefdrivenmultiagent}, and used cross-examination to detect errors \citep{cohen2023lmvslm}, while \citet{li2024moreagents} and \citet{chen2024scaling} studied how performance scales with the number of agents. The iterative belief exchange in such systems recalls DeGroot averaging \citep{degroot1974reaching} and its bounded-confidence variants \citep{deffuant2000mixing, hegselmann2002opinion}, whose clustering and polarization phenomena provide an analogy for LLM debate dynamics; \citet{baccelli2021steady} characterize stationary distributions for stochastic opinion dynamics under power-law confidence bounds. Prior work asks \emph{how} agents should deliberate or \emph{when} debate should stop; neither asks the deployment question we address: \emph{when should the system be allowed to act, and when should it refuse?}


\paragraph{Conformal prediction for LLMs.}
Conformal prediction \citep{vovk2005algorithmic, angelopoulos2023gentle} constructs distribution-free prediction sets; key methods include APS \citep{romano2020classification}, RAPS \citep{angelopoulos2021uncertainty}, and rank-based scores \citep{huang2023conformalranking}.
In NLP, \citet{kumar2023conformal} applied it to multiple-choice QA, \citet{quach2024conformal} to open-ended generation, and \citet{su2024apiconformal} to black-box settings matching ours.
The most related work is Debate-as-Optimization \citep[DAO;][]{wang2024debateoptimizationadaptiveconformal}, which uses conformal prediction as an \emph{intra-debate filter} with an adaptively decaying threshold ($\qhat_t = \beta \cdot \qhat_{t-1}$). This adaptive update falls outside the standard fixed-threshold split-conformal guarantee.
We instead apply standard split conformal \emph{post-hoc} to the collective belief at each pre-specified round, preserving a clean marginal guarantee in a fully black-box setting.

\paragraph{Verbalized confidence and social choice.}
A growing line of work elicits uncertainty verbally \citep{lin2022teachingmodels, kadavath2022languagemodels, tian2023justaskcalibrationstrategies, xiong2024llmsexpressuncertainty, yang2024verbalizedconfidence}; semantic uncertainty \citep{kuhn2023semantic} and self-consistency \citep{wang2023selfconsistency} provide complementary signals, but all remain single-agent methods.
On the social choice side, \citet{Yang_2024} and \citet{choi2025debatevoteyieldsbetter} treat aggregation as winner selection; we instead aggregate verbalized probabilities from \emph{multiple heterogeneous} agents via the linear opinion pool \citep{genest1986combining} and apply conformal calibration post-hoc, extending social choice to set-valued prediction with marginal coverage control without assuming any individual agent is well-calibrated.

\section{Method: Conformal Social Choice}
\label{sec:method}

We present the Conformal Social Choice framework, a four-stage pipeline that transforms the outputs of a multi-agent debate into prediction sets with marginal coverage guarantees (i.e., population-level, not per-instance).
Our goal is not to introduce new conformal theory, but to provide a formal contract for act-versus-escalate decisions in multi-agent debate.
Figure~\ref{fig:pipeline} illustrates the architecture.

\subsection{Problem Formulation}
\label{sec:problem}

Consider a multiple-choice question-answering task with input space $\calX$ and finite label space $\calY = \{A, B, C, D, \ldots\}$.
MMLU-Pro has $|\calY|{=}10$, while LogiQA and GPQA have $|\calY|{=}4$; extension to open-ended generation is discussed in the Limitations section.
We have access to an ensemble of $N$ LLM agents $\{M_1, \ldots, M_N\}$ that engage in a $T$-round debate.
Given a held-out calibration set $\calD_{\text{cal}} = \{(x_i, y_i)\}_{i=1}^{n}$ with ground-truth labels, our goal is to construct a set-valued predictor $\calC: \calX \to 2^{\calY}$ satisfying the \emph{marginal coverage} guarantee:
\begin{equation}
    \label{eq:coverage}
    \Pr\bigl[Y_{\text{test}} \in \calC(X_{\text{test}})\bigr] \geq 1 - \alpha,
\end{equation}
where $(X_{\text{test}}, Y_{\text{test}})$ is a new exchangeable test example and $\alpha \in (0,1)$ is the user-specified miscoverage rate.
The output $\calC(x)$ is not merely a set for evaluation; it is the object used for downstream \emph{action}: a singleton triggers automation, a larger set triggers human escalation, and an empty set flags an anomaly (\S\ref{sec:action_policy}).
This formulation recasts multi-agent debate from an accuracy-maximization problem into a \emph{decision} problem with fixed-round marginal coverage control.

\subsection{Stage 1: Multi-Agent Debate with Verbalized Probability Elicitation}
\label{sec:debate}

\paragraph{Agent ensemble and debate protocol.}
We employ a heterogeneous ensemble of $N$ LLM agents to promote diversity of reasoning strategies \citep{liang2023encouragingdivergentthinking}.
Given input $x$, the debate proceeds for $T$ rounds.
At each round $t \in \{0, 1, \ldots, T{-}1\}$, every agent $M_i$ receives the question $x$ along with a summary of all agents' responses from round $t{-}1$ (empty for $t{=}0$) and produces a \emph{verbalized probability distribution}:
\begin{equation}
    \label{eq:verbalized}
    \pi_i^{(t)}(y \mid x) \quad \text{for all } y \in \calY,
\end{equation}
where $\pi_i^{(t)}(y \mid x) \geq 0$ and $\sum_{y \in \calY} \pi_i^{(t)}(y \mid x) = 1$.
Rather than relying on token-level log-probabilities (unavailable for many proprietary APIs), each agent is prompted to output explicit numerical probabilities within structured tags \citep{tian2023justaskcalibrationstrategies, xiong2024llmsexpressuncertainty}.
Parsed probabilities are post-processed (clipped to $[0,1]$, renormalized); when parsing fails entirely, a uniform distribution is substituted.
Verbalized scores are noisy proxies for true beliefs and can be sensitive to prompting choices \citep{yang2024verbalizedconfidence}.
Crucially, noisy inputs do not invalidate the conformal coverage guarantee of Theorem~\ref{thm:coverage} under exchangeability, although they can affect set \emph{efficiency}.
In practice, parsing failures are rare (0.77\%; Appendix~\ref{sec:parsing_failures}).

\paragraph{Debate dynamics.}
At each round $t \geq 1$, agents observe a summary of all other agents' reasoning and confidence distributions from round $t{-}1$, creating a deliberative process where agents update their beliefs in light of peer arguments \citep{du2023improvingfactualityreasoninglanguage}.
By passing only the immediately preceding round's summary (rather than the full history), we reduce context length while preserving the most recent state of deliberation.

\subsection{Stage 2: Social Probability Aggregation}
\label{sec:aggregation}

We aggregate agent beliefs into a \emph{full probability distribution}, rather than a single vote or argmax, so that conformal calibration (\S\ref{sec:calibration}) can use a label-wise non-conformity score.
We adopt a \emph{linear opinion pool} \citep{genest1986combining}, a weighted mixture standard in Bayesian aggregation and social choice theory.

\begin{definition}[Social Probability]
\label{def:social_prob}
Given agent distributions $\{\pi_i^{(t)}\}_{i=1}^{N}$ at round $t$ and agent weights $\{w_i\}_{i=1}^{N}$ with $\sum_i w_i = 1$, the \emph{social probability} for option $y$ is:
\begin{equation}
    \label{eq:social_score}
    \psocial^{(t)}(y \mid x) = \sum_{i=1}^{N} w_i \cdot \pi_i^{(t)}(y \mid x).
\end{equation}
\end{definition}

\begin{proposition}[Basic properties of social probability]
\label{prop:basic_properties}
The social probability $\psocial(\cdot \mid x)$ satisfies normalization, anonymity (under equal weights), neutrality, unanimity, and monotonicity.
\end{proposition}
\noindent All properties follow from the linearity of weighted sums; proof in Appendix~\ref{app:proofs}.
Unlike majority voting, which reduces each agent's output to a single vote, the social probability preserves the \emph{intensity} of preferences: an agent assigning 0.6 to option A contributes differently than one assigning 0.99, a distinction that majority voting loses entirely.

\paragraph{Weight strategy.}
We use uniform weights $w_i = 1/N$ throughout.
This is the cleanest assumption-free default: learned or accuracy-based weighting would require extra supervision or validation data, complicating the black-box, post-hoc framing.
Empirically, one entropy-based weighting rule (which upweights more confident agents per instance) produces small differences (mean $|\Delta\text{coverage}| = 0.6\%$, mean $|\Delta\text{set size}| = 0.06$; Appendix~\ref{app:weighting}), showing insensitivity to this tested weighting choice.

\paragraph{Robustness of the social winner.}
As a diagnostic property, we characterize when the top-ranked label is stable under agent-level perturbations.

\begin{theorem}[Margin robustness]
\label{thm:robustness}
Let $\Delta(x) = \psocial(y_1 \mid x) - \psocial(y_2 \mid x)$ be the margin between the top two labels. If perturbed distributions satisfy $\|\widetilde{\pi}_i(\cdot \mid x) - \pi_i(\cdot \mid x)\|_\infty \leq \varepsilon_i$, then
$\|\widetilde{P}_{\mathrm{social}}(\cdot \mid x) - \psocial(\cdot \mid x)\|_\infty \leq \sum_i w_i \varepsilon_i$.
If $\Delta(x) > 2\sum_i w_i \varepsilon_i$, the top label is unchanged.
\end{theorem}
\noindent Proof in Appendix~\ref{app:proofs}. Large margins imply robustness; the conformal calibration of Stage~3 complements this by providing a coverage guarantee regardless of margin size.

\subsection{Stage 3: Conformal Calibration}
\label{sec:calibration}

We apply split conformal prediction \citep{vovk2005algorithmic} on top of the aggregated social probabilities, transforming heuristic confidence scores into rigorous coverage guarantees.

\paragraph{Non-conformity score.}
We define a non-conformity score measuring how poorly a candidate label $y$ conforms to the social consensus:

\begin{definition}[Probability-based non-conformity score]
\label{def:prob_score}
\begin{equation}
    \label{eq:snc_prob}
    \snc(x, y) = 1 - \psocial(y \mid x).
\end{equation}
\end{definition}

\noindent This score is high when the social probability for $y$ is low, indicating that the collective believes $y$ is unlikely.
An alternative rank-based cumulative score (analogous to APS; \citealp{romano2020classification}) is defined in Appendix~\ref{app:rank_score}.
We use the probability-based score throughout because it yields a clean threshold interpretation (Proposition~\ref{prop:threshold}: $y \in \calC(x)$ iff $\psocial(y \mid x) \geq 1 - \qhat$), directly linking set membership to social probability and making the automate/escalate decision transparently interpretable.

\paragraph{Calibration procedure.}
Given the calibration set $\calD_{\text{cal}} = \{(x_i, y_i)\}_{i=1}^{n}$:
(1)~for each calibration example $(x_i, y_i)$, run the debate pipeline and compute the social probability $\psocial(\cdot \mid x_i)$;
(2)~compute the non-conformity score of the ground truth: $s_i = \snc(x_i, y_i)$;
(3)~determine the conformal threshold:
\begin{equation}
    \label{eq:qhat}
    \qhat = \mathrm{Quantile}\!\Big(s_1, \ldots, s_n;\; \tfrac{\lceil (n{+}1)(1{-}\alpha) \rceil}{n}\Big).
\end{equation}
The finite-sample correction $\lceil (n+1)(1-\alpha) \rceil / n$ ensures the following guarantee:

\begin{theorem}[Marginal Coverage \citep{vovk2005algorithmic}]
\label{thm:coverage}
If the calibration and test examples are exchangeable, then the prediction set $\calC(x) = \{y \in \calY : \snc(x, y) \leq \qhat\}$ satisfies:
\begin{equation}
    \Pr\bigl[Y_{\text{test}} \in \calC(X_{\text{test}})\bigr] \geq 1 - \alpha.
\end{equation}
\end{theorem}

This guarantee is \emph{distribution-free}: it requires only exchangeability of calibration and test data, with no assumptions on individual agent calibration.
We stress that this is a guarantee on \emph{set coverage}, not on per-instance correctness, singleton accuracy, or conditional coverage within any subgroup.

\paragraph{Fixed-round scope.}
We calibrate a separate threshold $\qhat_t$ at each pre-specified debate round $t$, so Theorem~\ref{thm:coverage} gives
$\Pr[Y_{\mathrm{test}}\in\calC_t(X_{\mathrm{test}})]\geq 1-\alpha$
for every fixed $t$ under exchangeability and an unchanged debate pipeline.
This guarantee does \emph{not} automatically extend to the data-dependent stopping time that selects the first round with $|\calC_t(x)|=1$.
We therefore report fixed-round coverage separately from empirical adaptive-stopping outcomes and make no coverage claim for the latter.
Conservative simultaneous calibration or anytime-valid sequential conformal methods are promising ways to close this gap.

\begin{proposition}[Threshold form of the prediction set]
\label{prop:threshold}
For $\snc^{\mathrm{prob}}(x,y) = 1 - \psocial(y \mid x)$, the conformal prediction set at threshold $\qhat$ satisfies:
\begin{align}
    \calC(x) &= \{y : \snc^{\mathrm{prob}}(x,y) \leq \qhat\} \notag \\
              &= \{y : \psocial(y \mid x) \geq 1 - \qhat\}.
\end{align}
\end{proposition}
\noindent The equivalence follows by rearranging the inequality (Appendix~\ref{app:proofs}).

\subsection{Stage 4: Prediction Set Construction and Action Policy}
\label{sec:action_policy}

For a new test input $x$, the system executes the debate, computes $\psocial(\cdot \mid x)$, and constructs the prediction set:
\begin{equation}
    \label{eq:pred_set}
    \calC(x) = \bigl\{y \in \calY : \snc(x, y) \leq \qhat\bigr\}.
\end{equation}

The set size $|\calC(x)|$ provides an operational measure of the ambiguity retained by the calibrated prediction set, bounded by $\min\{|\calY|, \lfloor 1/(1{-}\qhat) \rfloor\}$ when $\qhat<1$ (Corollary~\ref{cor:cardinality} in Appendix).

\begin{proposition}[Conditions for singleton automation]
\label{prop:action_policy}
Let $p_k = \psocial(y_k \mid x)$ with $p_1 \geq p_2 \geq \cdots$, $\Delta(x) = p_1 - p_2$, and $\tau = 1 - \qhat$. Then:
(1)~$|\calC(x)| = 1$ (automation) iff $p_1 \geq \tau$ and $p_2 < \tau$; equivalently, $p_1 \geq \tau$ and $\Delta(x) > p_1 - \tau$.
A simpler sufficient condition is $p_1 \geq \tau$ and $\Delta(x) > \qhat$.
(2)~$|\calC(x)| \geq 2$ (escalation) iff $p_2 \geq \tau$.
(3)~$|\calC(x)| = 0$ (full review) iff $p_1 < \tau$.
\end{proposition}
\noindent Proof in Appendix~\ref{app:proofs}. Intuitively, singleton automation requires the runner-up to fall below the inclusion threshold $\tau = 1 - \qhat$; when agents are split, the prediction set grows, triggering escalation.
These conditions yield the following \textbf{hierarchical action policy}:

\paragraph{Case 1: $|\calC(x)| = 1$ (Full Automation).}
The system outputs the single answer autonomously.
Singleton status does not guarantee per-instance correctness; the coverage guarantee is marginal, not conditional on set size.

\paragraph{Case 2: $|\calC(x)| > 1$ (Human-in-the-Loop Escalation).}
The prediction set contains multiple candidates, indicating residual ambiguity.
The system \emph{abstains} and escalates the pruned candidate set to a human expert, communicating both that it is uncertain and \emph{which options remain plausible}, reducing the decision space from $|\calY|$ to $|\calC(x)|$.

\paragraph{Case 3: $|\calC(x)| = 0$ (Anomaly Detection).}
No option conforms to the social consensus, triggering a full manual review.

\noindent This policy transforms calibrated uncertainty into an operational decision rule.
\emph{Important caveat:} the coverage guarantee is marginal and applies at each fixed round; it is not a per-instance guarantee and does not extend to adaptive stopping.

\section{Experimental Setup}
\label{sec:experiments}

\paragraph{Dataset.}
We evaluate primarily on \textbf{MMLU-Pro} \citep{wang2024mmlupro}, a 10-option multiple-choice benchmark covering eight professional domains: Mathematics ($n{=}1{,}351$), Physics ($n{=}1{,}299$), Chemistry ($n{=}1{,}132$), Law ($n{=}1{,}101$), Engineering ($n{=}969$), Economics ($n{=}844$), Health ($n{=}818$), and Psychology ($n{=}798$).
The expanded option space (compared to MMLU's 4 options) makes the task substantially harder and better differentiates methods.
We additionally evaluate the same procedure on \textbf{LogiQA} \citep{liu2020logiqa} ($n{=}8{,}678$; 4-option logical reasoning) and \textbf{GPQA Extended} \citep{rein2023gpqa} ($n{=}546$; 4-option graduate-level science).
A second agent ensemble is reported in Appendix~\ref{app:additional_generality}.
MMLU-Pro domain descriptions and all split sizes are provided in Appendix~\ref{app:setup}.

\paragraph{Baselines.}

We contextualize our set-valued framework against standard point-estimate methods, where debate + majority voting achieves the highest accuracy (66.7--94.2\% across domains), outperforming greedy decoding, self-reflection, and static majority voting (Appendix Table~\ref{tab:baselines} and \ref{app:setup}).
We use consensus stopping as the primary comparator because it reflects a common stopping rule in debate systems: commit to an answer once agents agree.
We also compare fixed and calibration-tuned top-probability thresholds and a calibration-tuned entropy threshold.
We do not claim uniformly better interception: these scalar rules can reach similar operating points, but they return only act/abstain decisions and lack conformal's finite-sample fixed-round marginal set-coverage guarantee.

\paragraph{Metrics.}
We report \emph{marginal coverage} (fraction of test instances where $y \in \calC(x)$; target $\geq 1{-}\alpha$), \emph{average set size} $\frac{1}{m}\sum_j |\calC(x_j)|$ (smaller is better, conditioned on coverage), \emph{singleton rate} (fraction with $|\calC(x)|{=}1$), and \emph{singleton accuracy} (accuracy among instances with $|\calC(x)|{=}1$).
Coverage and set size are evaluated at each fixed round.
Singleton rate, singleton accuracy, and average stopping round summarize the empirical adaptive first-singleton policy; no coverage guarantee is claimed for that policy.
Two readings of singleton accuracy appear in the tables and we distinguish them explicitly: the \emph{per-round} form, reported in Table~\ref{tab:conformal}, is measured only among the cases \emph{newly} resolved at that round, so its late-round cells rest on small denominators; the \emph{cumulative} form, reported in Table~\ref{tab:early_stopping} and used for all headline numbers, is measured over all cases resolved before the final round, excluding cases that first reach a singleton at the final round.

\paragraph{Implementation.}
Our ensemble consists of three heterogeneous LLMs accessed via AWS Bedrock: Claude Haiku~4.5 \citep{anthropic2025claudehaiku45}, DeepSeek-R1 \citep{deepseek2025r1}, and Qwen-3~32B \citep{qwen2025qwen3}, selected for architectural and training diversity.
Debate runs for $T{=}4$ rounds (indexed 0--3) with temperature 0.7.
For conformal calibration, each domain is split 50/50 into calibration and test sets, and we evaluate at $\alpha \in \{0.05, 0.10\}$.
The threshold $\qhat_t$ is computed independently per dataset and fixed round.
Full implementation details (prompting, parsing, hyperparameters) are in Appendix~\ref{app:setup}.

\section{Results}
\label{sec:results}


We organize results in three parts: (1)~fixed-round coverage and calibration (\S\ref{sec:conformal_results}), (2)~the failure mode of consensus stopping (\S\ref{sec:consensus_failure}), and (3)~adaptive stopping as a safety mechanism, including comparisons with scalar abstention rules (\S\ref{sec:conformal_safety}).

\begin{figure}[t]
  \centering
  \includegraphics[width=\columnwidth]{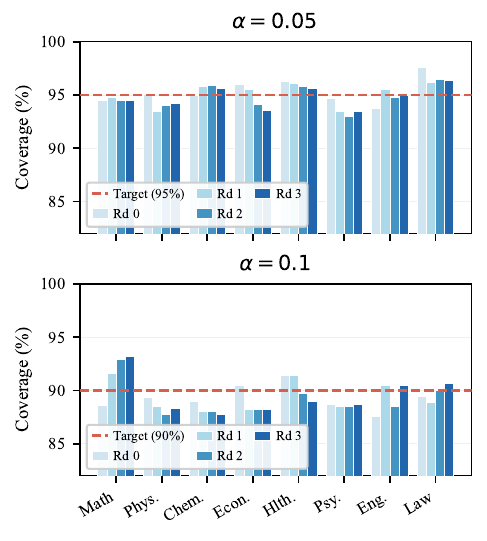}
  \caption{Coverage maintenance across debate rounds (indexed 0--3). At $\alpha{=}0.05$, coverage remains near 95\% across all four rounds for most domains. At $\alpha{=}0.10$, coverage clusters around the 90\% target. Values are the fixed-round coverages of Table~\ref{tab:conformal}.}
  \label{fig:coverage}
\end{figure}

\subsection{Conformal Prediction Sets: Coverage and Calibration}
\label{sec:conformal_results}

\begin{table*}[t]
\centering
\scriptsize
\setlength{\tabcolsep}{2.7pt}
\resizebox{\textwidth}{!}{%
\begin{tabular}{@{}cl cc cc cc cc cc cc cc cc cc cc@{}}
\toprule
& & \multicolumn{16}{c}{\textbf{MMLU-Pro ($|\calY|{=}10$)}} & \multicolumn{2}{c}{\textbf{LogiQA}} & \multicolumn{2}{c}{\textbf{GPQA}} \\
\cmidrule(lr){3-18} \cmidrule(lr){19-20} \cmidrule(lr){21-22}
& & \multicolumn{2}{c}{\textbf{Eng.}} & \multicolumn{2}{c}{\textbf{Law}} & \multicolumn{2}{c}{\textbf{Chem.}} & \multicolumn{2}{c}{\textbf{Phys.}} & \multicolumn{2}{c}{\textbf{Math}} & \multicolumn{2}{c}{\textbf{Econ.}} & \multicolumn{2}{c}{\textbf{Hlth.}} & \multicolumn{2}{c}{\textbf{Psy.}} & \multicolumn{2}{c}{\textbf{($|\calY|{=}4$)}} & \multicolumn{2}{c}{\textbf{($|\calY|{=}4$)}} \\
\cmidrule(lr){3-4} \cmidrule(lr){5-6} \cmidrule(lr){7-8} \cmidrule(lr){9-10} \cmidrule(lr){11-12} \cmidrule(lr){13-14} \cmidrule(lr){15-16} \cmidrule(lr){17-18} \cmidrule(lr){19-20} \cmidrule(lr){21-22}
\textbf{Metric} & \textbf{Rd} & \textbf{.05} & \textbf{.10} & \textbf{.05} & \textbf{.10} & \textbf{.05} & \textbf{.10} & \textbf{.05} & \textbf{.10} & \textbf{.05} & \textbf{.10} & \textbf{.05} & \textbf{.10} & \textbf{.05} & \textbf{.10} & \textbf{.05} & \textbf{.10} & \textbf{.05} & \textbf{.10} & \textbf{.05} & \textbf{.10} \\
\midrule
\multirow{4}{*}{$\qhat_t$}
  & 0 & .975 & .952 & .980 & .945 & .962 & .784 & .950 & .769 & .750 & .605 & .937 & .764 & .970 & .904 & .954 & .890 & .917 & .857 & .923 & .897 \\
  & 1 & .987 & .960 & .990 & .967 & .982 & .702 & .923 & .651 & .718 & .359 & .974 & .744 & .984 & .924 & .967 & .933 & .957 & .917 & .974 & .950 \\
  & 2 & .990 & .963 & .994 & .977 & .987 & .669 & .957 & .492 & .704 & .229 & .980 & .823 & .990 & .930 & .978 & .944 & .970 & .937 & .987 & .957 \\
  & 3 & .995 & .980 & .997 & .984 & .990 & .733 & .973 & .537 & .740 & .126 & .980 & .852 & .992 & .940 & .983 & .950 & .977 & .947 & .990 & .963 \\
\midrule
\multirow{4}{*}{\shortstack[l]{Fixed-rd.\\Cov. (\%)}}
  & 0 & 93.8 & 87.6 & 97.6 & 89.5 & 95.0 & 89.0 & 95.1 & 89.4 & 94.5 & 88.6 & 96.0 & 90.5 & 96.3 & 91.4 & 94.7 & 88.7 & 95.2 & 88.6 & 95.6 & 93.0 \\
  & 1 & 95.5 & 90.5 & 96.2 & 88.9 & 95.8 & 88.0 & 93.5 & 88.5 & 94.8 & 91.6 & 95.5 & 88.2 & 96.1 & 91.4 & 93.5 & 88.5 & 96.0 & 89.3 & 97.1 & 93.0 \\
  & 2 & 94.8 & 88.5 & 96.5 & 90.0 & 95.9 & 88.0 & 94.0 & 87.8 & 94.5 & 92.9 & 94.1 & 88.2 & 95.8 & 89.7 & 93.0 & 88.5 & 96.2 & 88.9 & 97.1 & 91.2 \\
  & 3 & 95.0 & 90.5 & 96.4 & 90.7 & 95.6 & 87.8 & 94.2 & 88.3 & 94.5 & 93.2 & 93.6 & 88.2 & 95.6 & 89.0 & 93.5 & 88.7 & 96.2 & 88.9 & 94.9 & 90.5 \\
\midrule
\multirow{4}{*}{$|\calC_t|$}
  & 0 & 4.40 & 3.08 & 6.99 & 3.76 & 2.57 & 1.37 & 2.09 & 1.33 & 1.26 & 0.92 & 1.87 & 1.23 & 3.73 & 1.71 & 2.59 & 1.50 & 2.81 & 1.82 & 2.86 & 2.56 \\
  & 1 & 3.92 & 2.40 & 6.49 & 3.65 & 1.91 & 1.06 & 1.26 & 1.02 & 1.06 & 0.96 & 1.82 & 1.08 & 3.53 & 1.56 & 2.23 & 1.49 & 2.92 & 1.89 & 2.81 & 2.45 \\
  & 2 & 3.63 & 2.00 & 6.79 & 3.65 & 1.75 & 1.02 & 1.27 & 1.00 & 1.02 & 0.97 & 1.73 & 1.08 & 3.93 & 1.42 & 2.27 & 1.44 & 2.98 & 1.87 & 2.82 & 2.22 \\
  & 3 & 3.84 & 2.13 & 6.91 & 3.94 & 1.81 & 1.01 & 1.27 & 1.00 & 1.01 & 0.98 & 1.65 & 1.06 & 4.13 & 1.38 & 2.45 & 1.44 & 3.00 & 1.88 & 2.75 & 2.14 \\
\midrule
\multirow{4}{*}{\shortstack[l]{Adaptive\\Sing. Rate}}
  & 0 & 26.8 & 32.0 &  1.5 &  6.4 & 43.3 & 65.7 & 47.8 & 69.2 & 74.9 & 92.0 & 52.1 & 77.3 & 20.3 & 50.4 & 39.6 & 62.7 & 14.6 & 38.8 & 17.6 & 23.4 \\
  & 1 & 45.2 & 56.5 &  4.0 & 14.3 & 70.8 & 92.8 & 81.8 & 97.5 & 94.1 & 95.6 & 65.4 & 91.5 & 28.4 & 60.6 & 49.6 & 70.9 & 22.7 & 46.0 & 29.3 & 38.1 \\
  & 2 & 50.5 & 64.3 &  5.5 & 18.5 & 77.0 & 97.5 & 86.0 & 99.5 & 97.6 & 97.5 & 70.1 & 93.1 & 31.6 & 68.9 & 52.4 & 74.7 & 24.8 & 48.4 & 31.9 & 48.7 \\
  & 3 & 52.6 & 67.2 &  6.2 & 20.9 & 78.8 & 98.8 & 87.2 & 99.8 & 98.2 & 97.8 & 73.9 & 94.3 & 34.2 & 72.4 & 54.1 & 76.2 & 25.7 & 49.6 & 35.2 & 51.6 \\
\midrule
\multirow{4}{*}{\shortstack[l]{Adaptive\\Sing. Acc.}}
  & 0 & 96.9 & 96.8 & 100 & 91.4 & 97.5 & 93.8 & 98.4 & 92.2 & 96.4 & 96.1 & 97.3 & 91.7 & 91.6 & 92.7 & 96.8 & 92.4 & 91.0 & 87.6 & 95.8 & 93.8 \\
  & 1 & 96.6 & 92.4 & 100 & 81.8 & 95.5 & 78.4 & 92.8 & 81.0 & 89.2 & 87.5 & 87.5 & 73.3 & 100 & 88.1 & 87.5 & 84.8 & 89.1 & 84.6 & 90.6 & 82.5 \\
  & 2 & 84.6 & 79.0 & 62.5 & 69.6 & 97.1 & 77.8 & 88.9 & 53.9 & 83.3 & 69.2 & 75.0 & 57.1 & 84.6 & 79.4 & 90.9 & 93.3 & 80.7 & 75.0 & 100 & 89.7 \\
  & 3 & 70.0 & 92.9 & 75.0 & 92.3 & 90.0 & 42.9 & 100 & 100 & 100 & 100 & 75.0 & 80.0 & 72.7 & 71.4 & 100 & 83.3 & 81.1 & 74.5 & 66.7 & 75.0 \\
\bottomrule
\end{tabular}}
\caption{Primary-ensemble results on MMLU-Pro, LogiQA, and GPQA. Coverage and $|\calC_t|$ are evaluated over all examples at each pre-specified round and inherit the fixed-round guarantee. Adaptive singleton rate is cumulative through each round; singleton accuracy is measured among cases newly resolved at that round. Neither stopping metric carries a coverage guarantee. Each dataset uses the same procedure with its own calibration split; no heuristic threshold tuning is used.}
\label{tab:conformal}
\end{table*}

Table~\ref{tab:conformal} presents the core results at $\alpha = 0.05$ and $\alpha = 0.10$.

\paragraph{Fixed-round coverage is close to target.}
At $\alpha{=}0.05$, empirical coverage ranges from 93.0\% to 97.6\% across MMLU-Pro, LogiQA, GPQA, and rounds; at $\alpha{=}0.10$, it ranges from 87.6\% to 93.2\%.
Repeating the MMLU-Pro calibration--test split 20 times leaves this picture unchanged: mean fixed-round coverage stays within 1.3 percentage points of the nominal target at both levels, with a standard deviation of at most 2.4 points across all domain--round combinations (Appendix~\ref{app:repeated_splits}).
Slight undercoverage on individual domain-round combinations is expected in finite-sample split conformal prediction ($O(1/\sqrt{n})$ deviation).

\paragraph{Set size is difficulty-adaptive.}
At $\alpha{=}0.05$, Math achieves average set size 1.01 (round~3), enabling confident automation on nearly all questions, while Law produces 6.91, consistent with lower accuracy and more dispersed social probabilities.
This domain-adaptive behavior emerges from calibration: under the empirical stopping policy, only 52.6\% of Engineering samples reach singleton status, compared with 98.2\% on Math.

\paragraph{Debate as uncertainty reduction.}
Changes in fixed-round $|\calC_t|$ show how uncertainty evolves through deliberation.
At $\alpha{=}0.05$, the singleton rate on Physics grows from 47.8\% (round~0) to 87.2\% (round~3), meaning debate resolves an additional 39.4\% of samples while keeping coverage on track.
This contrasts with point-estimate majority voting (Appendix~\ref{sec:baseline_results}), where accuracy plateaus after round~1 and provides no uncertainty quantification.
The singleton-rate statement is descriptive of adaptive stopping, not a coverage claim.
LogiQA and GPQA show the same qualitative safety--automation trade-off with dataset-specific calibration.

\subsection{Consensus Stopping and Wrong-Consensus Risk}
\label{sec:consensus_failure}

A key practical question is when to stop the debate.
The natural heuristic, \emph{consensus-based} stopping (terminate when all agents agree), is fast but fundamentally unsafe, because debate introduces a systematic failure mode: agents may unanimously agree on the \emph{wrong} answer through social reinforcement, which we term \emph{wrong-consensus convergence}.

\paragraph{Consensus stopping is fast but error-prone.}
Consensus-based early stopping terminates extremely early (average round 1.36--1.81 across domains, in the same 1-indexed resolved-only convention as Table~\ref{tab:early_stopping}), resolving 87.7--97.8\% of samples as singletons by the final round (Table~\ref{tab:early_stopping}). If unanimity is not reached, the baseline acts on the final-round point estimate rather than escalating (Figure~\ref{fig:headline}: 0\% escalation).
However, accuracy on pre-final-round consensus-stopped cases varies widely (67.9--94.9\%), and consensus provides \emph{no coverage guarantee}.
In Engineering, consensus stops samples by round 1.80 on average with 82.6\% accuracy, meaning 17.4\% of ``confidently'' stopped predictions are wrong, with no mechanism to flag them.

\paragraph{Wrong-consensus convergence explains why consensus fails.}
Among the 1{,}963 inference samples where agents initially \emph{disagree} at round~0, 469 (23.9\%) converge to unanimous wrong consensus by round~3: nearly one in four initially-disputed questions (Table~\ref{tab:sycophancy}).
The risk varies by domain: Law (34.8\%) and Psychology (33.3\%) are most susceptible, while Math (12.1\%) is most resilient.
Including cases where agents already agreed on a wrong answer at round~0, total wrong consensus at round~3 is 586 out of 4{,}158 (14.1\%). Consensus stopping treats all unanimous agreement equally, and therefore \emph{commits} to these wrong-consensus errors with no recourse.

\subsection{Conformal Stopping as a Safety Mechanism}
\label{sec:conformal_safety}

\begin{figure*}[t]
  \centering
  \includegraphics[width=\textwidth]{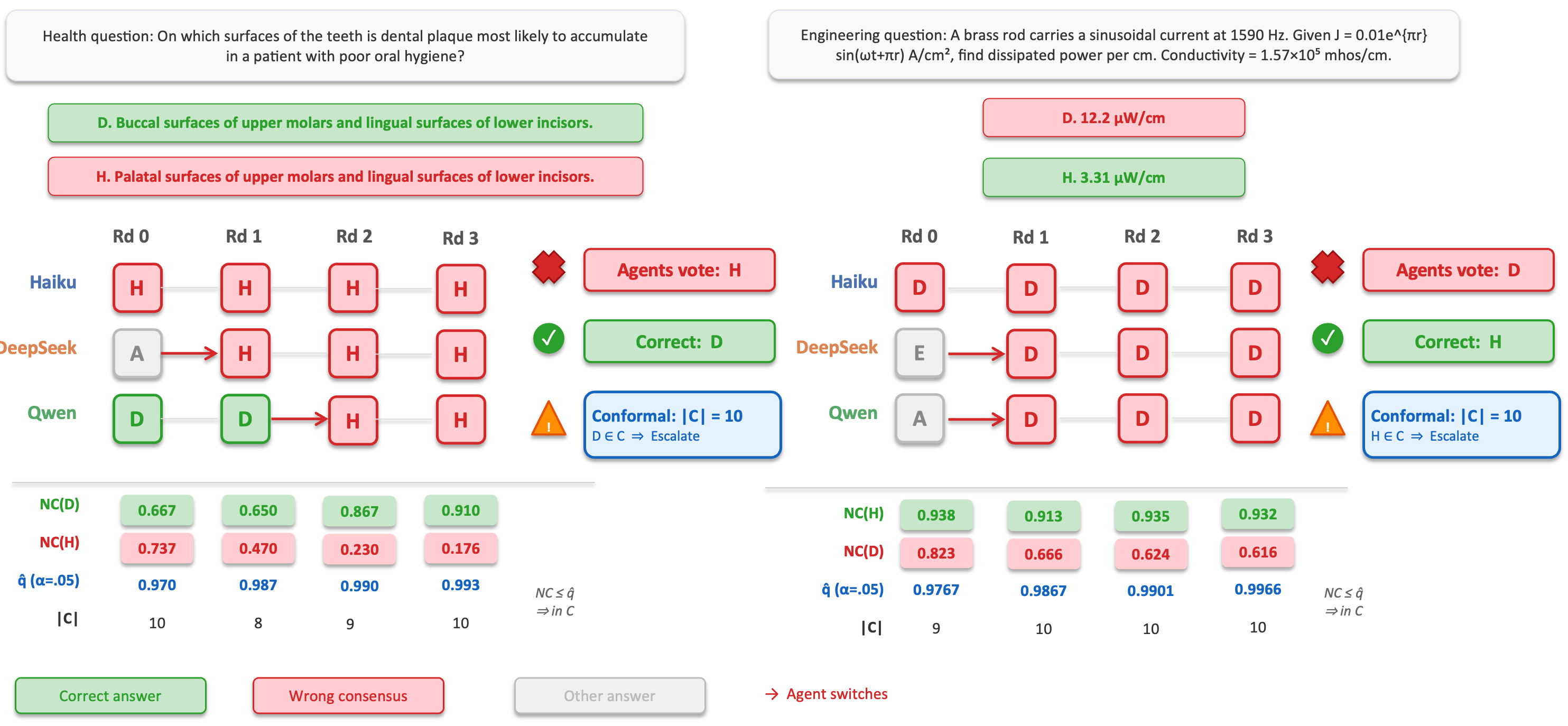}
  \caption{Two wrong-consensus examples at $\alpha{=}0.05$. Consensus acts on the unanimous wrong answer, whereas the fixed-round conformal set retains the correct alternative and triggers escalation.}
  \label{fig:sycophancy_rejection}
\end{figure*}

\begin{table}[t]
\centering
\resizebox{\columnwidth}{!}{%
\setlength{\tabcolsep}{2pt}
\begin{tabular}{@{}l r cc cc cc r@{}}
\toprule
& & \multicolumn{2}{c}{\textbf{Avg Round}} & \multicolumn{2}{c}{\textbf{Singleton \%}} & \multicolumn{2}{c}{\textbf{Singleton Acc.}} & \\
\cmidrule(lr){3-4} \cmidrule(lr){5-6} \cmidrule(lr){7-8}
\textbf{Domain} & $n$ & \textbf{Consensus} & \textbf{Conformal} & \textbf{Consensus} & \textbf{Conformal} & \textbf{Consensus} & \textbf{Conformal} & $\Delta$\textbf{Acc} \\
\midrule
Engineering & 485 & 1.80 & 1.67 & 90.1 & 52.6 & 82.6 & 95.5 & +12.9 \\
Law         & 551 & 1.81 & 2.24 & 87.7 &  6.2 & 67.9 & 90.0 & +22.1 \\
Chemistry   & 566 & 1.63 & 1.57 & 95.2 & 78.8 & 90.0 & 96.8 & +6.8 \\
Physics     & 650 & 1.55 & 1.53 & 95.1 & 87.2 & 90.9 & 95.7 & +4.8 \\
Math        & 676 & 1.43 & 1.29 & 97.8 & 98.2 & 94.9 & 94.6 & $-$0.3 \\
Economics   & 422 & 1.36 & 1.46 & 96.0 & 73.9 & 87.9 & 93.9 & +6.0 \\
Health      & 409 & 1.50 & 1.66 & 93.4 & 34.2 & 84.8 & 93.0 & +8.2 \\
Psychology  & 399 & 1.39 & 1.38 & 96.7 & 54.1 & 83.7 & 94.7 & +11.0 \\
\midrule
\multicolumn{9}{@{}l}{\emph{Additional benchmarks}} \\
\midrule
LogiQA & 4339 & 1.51 & 2.38 & 97.1 & 25.7 & 75.6 & 89.2 & +13.7 \\
GPQA (ext.) & 273 & 1.84 & 2.21 & 93.4 & 35.2 & 75.3 & 91.7 & +16.3 \\
\bottomrule
\end{tabular}}
\caption{Empirical adaptive stopping ($\alpha{=}0.05$). Conformal first-singleton stopping resolves fewer cases but yields higher accuracy among resolved cases, including on LogiQA and GPQA. \textbf{Avg Round}: mean first-singleton round (1-indexed) over resolved samples. \textbf{Singleton \%}: cumulative rate over all samples. \textbf{Singleton Acc.}: accuracy over samples reaching a singleton \emph{before} the final round (same convention for both rules), so its denominator excludes final-round stops. These outcomes do not inherit the fixed-round coverage guarantee.}
\label{tab:early_stopping}
\end{table}

\begin{table}[t]
\centering
\resizebox{\columnwidth}{!}{%
\setlength{\tabcolsep}{2pt}
\begin{tabular}{@{}l r rrr rr rr@{}}
\toprule
& & \multicolumn{3}{c}{\textbf{Wrong-Consensus Convergence}} & \multicolumn{2}{c}{\textbf{Wrong Consensus}} & \multicolumn{2}{c}{\textbf{Correct Consensus}} \\
\cmidrule(lr){3-5} \cmidrule(lr){6-7} \cmidrule(lr){8-9}
\textbf{Domain} & $n$ & \textbf{Disagree} & \textbf{$\to$WC} & \textbf{WC\%} & \textbf{Count} & \textbf{Rejected\%} & \textbf{Count} & \textbf{Rejected\%} \\
\midrule
Engineering & 485 & 312 &  75 & 24.0 &  83 & 88.0 & 375 & 41.1 \\
Law         & 551 & 388 & 135 & 34.8 & 162 & 97.5 & 344 & 95.9 \\
Chemistry   & 566 & 288 &  55 & 19.1 &  63 & 82.5 & 494 & 18.8 \\
Physics     & 650 & 313 &  50 & 16.0 &  58 & 70.7 & 570 &  8.4 \\
Math        & 676 & 256 &  31 & 12.1 &  35 & 11.4 & 631 &  0.6 \\
Economics   & 422 & 135 &  35 & 25.9 &  50 & 62.0 & 362 & 26.5 \\
Health      & 409 & 145 &  46 & 31.7 &  67 & 88.1 & 335 & 69.3 \\
Psychology  & 399 & 126 &  42 & 33.3 &  68 & 91.2 & 323 & 43.3 \\
\midrule
\textbf{All} & \textbf{4{,}158} & \textbf{1{,}963} & \textbf{469} & \textbf{23.9} & \textbf{586} & \textbf{81.9} & \textbf{3{,}434} & \textbf{31.9} \\
\midrule
\multicolumn{9}{@{}l}{\emph{Additional benchmarks}} \\
\midrule
LogiQA & 4339 & 1810 & 635 & 35.1 & 1042 & 90.8 & 3139 & 74.2 \\
GPQA (ext.) & 273 & 190 & 55 & 28.9 & 63 & 87.3 & 190 & 61.6 \\
\bottomrule
\end{tabular}}
\caption{Wrong-consensus convergence and \emph{final-round} conformal interception at round~3 ($\alpha{=}0.05$). \textbf{$\to$WC}~=~converge to unanimous wrong consensus; \textbf{Rejected\%}~=~fraction whose round-3 conformal set satisfies $|\calC_3|>1$. This is a retrospective diagnostic on final-round consensus cases and is distinct from the adaptive first-singleton policy of Table~\ref{tab:early_stopping}. We use the conservative event $|\calC_3|>1$; empty-set anomaly cases ($|\calC_3|=0$), handled separately by the action policy, are not counted here. MMLU-Pro interception is 81.9\%; LogiQA and GPQA show 90.8\% and 87.3\%.}
\label{tab:combined_stopping}\label{tab:sycophancy}
\end{table}

\begin{table}[t]
\centering
\small
\setlength{\tabcolsep}{4pt}
\begin{tabular}{@{}lcc@{}}
\toprule
\textbf{Method} & \textbf{WC Rej.} & \textbf{CC Rej.} \\
\midrule
Conformal set & \textbf{81.9} & 31.9 \\
Fixed top-1 $\tau{=}.95$ & 79.4 & \textbf{26.8} \\
Top-1 (cal-tuned) & 79.7 & 33.5 \\
Entropy (cal-tuned) & 80.0 & 35.1 \\
\bottomrule
\end{tabular}

\caption{MMLU-Pro round-3 abstention comparison (\%) on the same final-round consensus cases as Table~\ref{tab:combined_stopping}. Scalar rules attain comparable wrong-consensus (WC) interception. Correct-consensus (CC) rejection measures unnecessary escalation. Full domain results are in Appendix~\ref{app:abstention_baselines}.}
\label{tab:abstention_main}
\end{table}


Table~\ref{tab:early_stopping} reports the empirical behavior of the adaptive first-singleton policy: conformal stopping yields higher singleton accuracy than consensus stopping, at the cost of resolving fewer cases automatically.
\emph{Separately}, Table~\ref{tab:combined_stopping} provides a \emph{final-round diagnostic}: among the 586 samples that end in unanimous wrong consensus at round~3, 81.9\% have $|\calC_3|>1$ at $\alpha{=}0.05$, while only 31.9\% of correct-consensus cases are over-rejected.
This diagnostic helps explain why conformal singletons are more accurate than consensus-stopped cases, but it is not a directly measured error-prevention rate for the adaptive policy: a sample that ends in wrong consensus at round~3 may already have been resolved as a singleton at an earlier round under the adaptive rule.
The singleton-accuracy gain in Table~\ref{tab:early_stopping} should therefore not be read as a generic reasoning improvement; it arises mainly because the conformal layer refuses to act on many cases where debate reaches unanimous but wrong consensus.

\paragraph{Conformal stopping selects reliable predictions.}
Among cases resolved before the final round, conformal singleton accuracy is 90.0--96.8\% across all eight domains (Table~\ref{tab:early_stopping}), up to \textbf{22.1 percentage points higher} than consensus stopping (Law: 90.0\% vs.\ 67.9\%).
This improvement is largest on domains where wrong-consensus convergence is most prevalent (\S\ref{sec:consensus_failure}), consistent with the interpretation that conformal stopping flags the uncertain cases that consensus would commit to.

\paragraph{Final-round wrong-consensus interception.}
Across all 586 samples that end in wrong consensus at round~3, the round-3 conformal set is non-singleton on 81.9\% at $\alpha{=}0.05$ (Table~\ref{tab:combined_stopping}; detailed breakdown in Appendix~\ref{app:sycophancy}).
Figure~\ref{fig:sycophancy_rejection} illustrates two concrete examples.
Conversely, only 2 out of 4{,}158 samples become a wrong singleton at round~3 while agents do \emph{not} unanimously agree on that wrong answer (0.05\%), giving a 240:1 ratio between final-round wrong-consensus cases flagged and newly introduced wrong singletons (Appendix~\ref{app:sycophancy}).
The coverage guarantee of Theorem~\ref{thm:coverage} is on \emph{set coverage} at the population level, not per-instance correctness; nevertheless, calibration at level $\alpha$ empirically coincides with substantially higher singleton accuracy than consensus stopping. The gains in singleton accuracy come at the cost of automation coverage: on Math ($-$0.3pp), conformal stopping offers no advantage, while on Law it resolves only 6.2\% of samples, escalating the remaining 93.8\% to human review.
Law illustrates calibrated refusal under low score concentration: the method abstains often, and the operating point is user-adjustable via $\alpha$ (detailed domain-level analysis in Appendix~\ref{app:domain_analysis}).
The same calibration layer can also be applied to other multi-agent systems that produce per-option confidence estimates.

\paragraph{Scalar abstention is competitive.}
Table~\ref{tab:abstention_main} shows that on the same final-round wrong-consensus cases, fixed top-probability and calibration-tuned top-1 and entropy rules flag 79.4--80.0\%, close to conformal's 81.9\%.
Thus, conformal prediction does not uniquely determine a superior interception--automation curve.
Its distinct output is a calibrated \emph{set}: at each fixed round it provides a finite-sample marginal coverage guarantee and, upon escalation, exposes the specific candidate alternatives that remain plausible rather than only a binary refusal.
This supports a human workflow in which the reviewer receives both a coverage-controlled candidate set and a reduced decision space.
Detailed domain results, including cases where scalar thresholds are more selective, appear in Appendix~\ref{app:abstention_baselines}.

\section{Conclusion}
\label{sec:conclusion}

We presented Conformal Social Choice, a post-hoc layer that converts multi-agent debate outputs into candidate sets with fixed-round marginal coverage control. By aggregating verbalized probabilities via a linear opinion pool and calibrating with split conformal prediction, the framework decides when to act autonomously and when to defer to human review, without retraining or access to model internals.

Across MMLU-Pro, LogiQA, and GPQA, the empirical first-singleton policy raises accuracy among the cases on which it acts. Separately, a final-round diagnostic shows that on the eight MMLU-Pro domains, 81.9\% of samples that end in wrong consensus at round~3 have non-singleton conformal sets at $\alpha{=}0.05$, and cases reaching a singleton before the final round achieve 90.0--96.8\% accuracy: a selection effect of calibrated refusal, not a reasoning improvement. In domains such as Law, where calibrated sets often remain non-singleton, most cases are escalated to human review; we view this as a feature of calibrated refusal.

While scalar abstention rules can match interception, the distinctive contribution is the set-valued interface: a fixed-round marginal coverage guarantee plus explicit alternatives for human escalation. Treating multi-agent debate as a decision problem, rather than answer selection, points toward safer deployment of LLM ensembles. Coverage is not guaranteed after adaptive stopping; anytime-valid sequential methods are a natural next step.


\section*{Limitations}
\label{sec:limitations}

\paragraph{Marginal vs.\ conditional coverage.}
The coverage guarantee of Theorem~\ref{thm:coverage} is \emph{marginal}: it holds on average over the test distribution but does not guarantee coverage for every individual instance or subgroup.
Achieving conditional coverage (e.g., coverage $\geq 1{-}\alpha$ for every difficulty level) requires additional techniques such as conformal risk control, group-balanced calibration, or Bayesian posterior bounds on conditional singleton error rates, which we leave to future work.

\paragraph{Fixed-round vs.\ adaptive validity.}
The guarantee holds separately for each pre-specified round.
Stopping at the first singleton is data-dependent, so the reported stopping results are empirical and do not inherit Theorem~\ref{thm:coverage}.
Simultaneous corrections or sequential conformal methods are needed for anytime-valid stopping.

\paragraph{Exchangeability assumption.}
Conformal prediction requires that calibration and test data be exchangeable.
In practice, this assumption may be violated by distribution shift (e.g., calibrating on one exam year and testing on another).
We mitigate this through random splitting, but deployment in non-stationary environments would benefit from online conformal methods.

\paragraph{Verbalized probabilities as noisy proxies.}
Eliciting full probability distributions from agents requires structured prompting and careful parsing, adding overhead compared to simple vote extraction.
Verbalized probabilities may be prompt-sensitive.
Parsing failures default to uniform distributions that may dilute the social probability signal.
Importantly, while these issues may affect set \emph{efficiency}, they do not invalidate the marginal coverage guarantee, which holds for any non-conformity score under exchangeability.

\paragraph{Computational cost.}
Running a 3-agent, 4-round debate requires 12 LLM inference calls per question.
While the conformal calibration itself is computationally trivial (quantile computation over scores), the debate overhead may be prohibitive for latency-sensitive applications.
The empirical first-singleton policy can reduce inference for samples that reach a singleton before the final round; the realized savings depend on the domain and on the fraction of samples that ever reach a singleton (high on Math, $>98\%$, and very low on Law, $6.2\%$ at $\alpha{=}0.05$).

\paragraph{Closed-set evaluation.}
Our experiments are limited to closed-set multiple-choice tasks with four or ten answer options; on tasks with even fewer options (e.g., binary classification), non-singleton sets would carry less informational value.
Extension to open-ended generation, where the label space is unbounded, would require additional design choices (e.g., candidate generation followed by conformal filtering) and is left to future work.

\section*{Ethics Statement}
This work studies calibrated abstention and human escalation for multi-agent LLM systems. The experiments use public multiple-choice benchmarks and involve no human subjects or personal data. The method does not guarantee per-instance or subgroup safety; high-stakes deployment requires domain-specific validation, distribution-shift monitoring, and qualified human oversight. Multi-agent inference also incurs computational and environmental costs.

\section*{Acknowledgments}
LLMs were used solely for language polishing and consistency checking of author-written text. All scientific content, analyses, and conclusions were developed and verified by the authors, who retain full responsibility for the final manuscript.


\bibliography{reference}

\appendix
\section{Proofs}
\label{app:proofs}

\begin{proof}[\textbf{Proof of Proposition~\ref{prop:basic_properties}}]
All five properties follow from the linearity of the weighted sum $\psocial(y \mid x) = \sum_{i=1}^{N} w_i \, \pi_i(y \mid x)$.

\medskip\noindent\textbf{(1) Normalization.}\;
Because each $\pi_i$ is a valid distribution,
\begin{align}
    \textstyle\sum_{y} \psocial(y \mid x)
    &= \textstyle\sum_{y} \sum_{i} w_i \pi_i(y \mid x) \notag \\
    &= \textstyle\sum_{i} w_i \!\!\underbrace{\textstyle\sum_{y} \pi_i(y \mid x)}_{\displaystyle =\,1}\!\!
    = 1. \notag
\end{align}
Non-negativity is immediate since $w_i \geq 0$ and $\pi_i(y \mid x) \geq 0$.

\medskip\noindent\textbf{(2) Anonymity.}\;
Under equal weights $w_i = 1/N$, the social probability becomes $\psocial(y \mid x) = \frac{1}{N}\sum_{i} \pi_i(y \mid x)$, which is invariant to any permutation of the agent indices.

\medskip\noindent\textbf{(3) Neutrality.}\;
Let $\rho : \calY \to \calY$ be any relabeling of answer options. Under $\rho$, each agent's distribution transforms as $\pi_i^{\rho}(\rho(y) \mid x) = \pi_i(y \mid x)$. Therefore,
\begin{align}
    P_{\mathrm{social}}^{\rho}\!(\rho(y) \mid x)
    &= \textstyle\sum_{i} w_i \pi_i^{\rho}(\rho(y) \mid x) \notag \\
    &= \textstyle\sum_{i} w_i \pi_i(y \mid x) \notag \\
    &= \psocial(y \mid x). \notag
\end{align}

\medskip\noindent\textbf{(4) Unanimity.}\;
If $\pi_i(y^\star \mid x) \geq \pi_i(y \mid x)$ for all agents $i$ and all $y \neq y^\star$, then
\begin{align}
    &\psocial(y^\star \mid x) - \psocial(y \mid x) \notag \\
    &\quad = \textstyle\sum_{i} w_i \bigl[\pi_i(y^\star \mid x) - \pi_i(y \mid x)\bigr]
    \geq 0. \notag
\end{align}

\medskip\noindent\textbf{(5) Monotonicity.}\;
Suppose agent $j$'s probability on $y^\star$ increases by $\delta > 0$, with compensating decreases redistributed among labels $y \neq y^\star$ so that $\pi_j$ remains normalized, and no other agent's distribution changes. Then $\psocial(y^\star \mid x)$ increases by $w_j \delta > 0$, since the perturbation to the linear pool is $w_j \delta$ on $y^\star$ and no other agent contributes to the change at $y^\star$.
\end{proof}

\begin{proof}[\textbf{Proof of Theorem~\ref{thm:robustness}}]
\emph{Part 1: $\ell_\infty$ bound.}\;
For each label $y \in \calY$, the triangle inequality gives
\begin{align}
    &\bigl|\widetilde{P}_{\mathrm{social}}(y \mid x) - \psocial(y \mid x)\bigr| \notag \\
    &\;= \bigl|\textstyle\sum_{i} w_i [\widetilde{\pi}_i(y \mid x) - \pi_i(y \mid x)]\bigr| \notag \\
    &\;\leq \textstyle\sum_{i} w_i \, |\widetilde{\pi}_i(y \mid x) - \pi_i(y \mid x)| \notag \\
    &\;\leq \textstyle\sum_{i} w_i \, \varepsilon_i. \notag
\end{align}
Taking the maximum over $y$ yields $\|\widetilde{P}_{\mathrm{social}}(\cdot \mid x) - \psocial(\cdot \mid x)\|_\infty \leq \sum_i w_i \varepsilon_i$.

\medskip\noindent\emph{Part 2: Winner stability.}\;
Let $y_1, y_2$ be the top two labels under $\psocial$. Since $y_2$ is the second-largest, $\psocial(y_1 \mid x) - \psocial(y \mid x) \geq \Delta(x)$ for every $y \neq y_1$. Combined with Part~1,
\begin{align}
    &\widetilde{P}_{\mathrm{social}}(y_1 \mid x) - \widetilde{P}_{\mathrm{social}}(y \mid x) \notag \\
    &\;\geq \bigl[\psocial(y_1 \mid x) - \textstyle\sum_i w_i \varepsilon_i\bigr] \notag \\
    &\;\quad - \bigl[\psocial(y \mid x) + \textstyle\sum_i w_i \varepsilon_i\bigr] \notag \\
    &\;\geq \Delta(x) - 2\textstyle\sum_i w_i \varepsilon_i. \notag
\end{align}
When $\Delta(x) > 2\sum_i w_i \varepsilon_i$, this quantity is strictly positive for every $y \neq y_1$, so $y_1$ remains the top label under the perturbed social probability.
\end{proof}

\begin{proof}[\textbf{Proof of Proposition~\ref{prop:threshold}}]
By definition, $y \in \calC(x)$ if and only if the non-conformity score satisfies
\begin{align}
    &\snc^{\mathrm{prob}}(x, y) \leq \qhat \notag \\
    &\;\iff\; 1 - \psocial(y \mid x) \leq \qhat \notag \\
    &\;\iff\; \psocial(y \mid x) \geq 1 - \qhat. \notag
\end{align}
Therefore $\calC(x) = \{y \in \calY : \psocial(y \mid x) \geq 1 - \qhat\}$.
\end{proof}

\begin{corollary}[Cardinality bound on the prediction set]
\label{cor:cardinality}
Under Proposition~\ref{prop:threshold}: if $\qhat < 1$, then
$|\calC(x)| \leq \min\bigl\{|\calY|,\; \lfloor 1/(1-\qhat) \rfloor\bigr\}$; if $\qhat = 1$, then $\calC(x) = \calY$.
\end{corollary}

\begin{proof}[\textbf{Proof of Corollary~\ref{cor:cardinality}}]
If $\qhat = 1$ then $1 - \qhat = 0$, so every label satisfies $\psocial(y \mid x) \geq 0$ and $\calC(x) = \calY$.
Otherwise assume $\qhat < 1$ and let $m = |\calC(x)|$. By Proposition~\ref{prop:threshold}, every label in $\calC(x)$ satisfies $\psocial(y \mid x) \geq 1 - \qhat$. Since $\psocial$ is a valid probability distribution (Proposition~\ref{prop:basic_properties}),
\[
    1 \;\geq\; \sum_{y \in \calC(x)} \psocial(y \mid x) \;\geq\; m \cdot (1 - \qhat).
\]
Rearranging gives $m \leq 1/(1 - \qhat)$, and since $m$ is an integer bounded by $|\calY|$, $m \leq \min\{|\calY|, \lfloor 1/(1 - \qhat) \rfloor\}$.
\end{proof}

\begin{proof}[\textbf{Proof of Proposition~\ref{prop:action_policy}}]
By Proposition~\ref{prop:threshold}, $\calC(x) = \{y : \psocial(y \mid x) \geq \tau\}$ where $\tau = 1 - \qhat$.
Write $p_k = \psocial(y_k \mid x)$ with $p_1 \geq p_2 \geq \cdots$.

\medskip\noindent\textbf{(1) Singleton condition.}
$|\calC(x)| = 1$ iff exactly one label meets the threshold $\tau$.
Since $p_1$ is the largest probability, this is equivalent to $p_1 \geq \tau$ and $p_2 < \tau$.
Using $\Delta(x) = p_1 - p_2$, the condition $p_2 < \tau$ rewrites as $p_1 - \Delta(x) < \tau$, i.e., $\Delta(x) > p_1 - \tau$.

For the simpler sufficient condition, suppose $\Delta(x) > \qhat = 1 - \tau$.
Then, since $p_1 \leq 1$:
\[
  p_2 = p_1 - \Delta(x) \leq 1 - \Delta(x) < 1 - \qhat = \tau.
\]
Therefore, if additionally $p_1 \geq \tau$, we have $|\calC(x)| = 1$.

\medskip\noindent\textbf{(2) Multiple candidates.}
$|\calC(x)| \geq 2$ iff at least two labels satisfy the inclusion criterion.
Since $p_2$ is the second-largest probability, this is equivalent to $p_2 \geq \tau$.

\medskip\noindent\textbf{(3) Empty set.}
$|\calC(x)| = 0$ iff no label meets the threshold.
Since $p_1 \geq p_k$ for all $k$, this is equivalent to $p_1 < \tau$.
\end{proof}

\section{Rank-Based Cumulative Score}
\label{app:rank_score}

As an alternative to the probability-based score (Definition~\ref{def:prob_score}), one can define a rank-based cumulative non-conformity score analogous to the Adaptive Prediction Sets (APS) score of \citet{romano2020classification}.

\begin{definition}[Rank-based cumulative score]
Let $\sigma$ be a permutation of $\calY$ such that $\psocial(\sigma(1) \mid x) \geq \psocial(\sigma(2) \mid x) \geq \cdots$. For option $y$ at rank $r$:
\begin{equation}
    \snc^{\mathrm{rank}}(x, y) = \sum_{j=1}^{r} \psocial(\sigma(j) \mid x).
\end{equation}
\end{definition}

\noindent The rank-based score captures ordinal information: even if two options have similar probabilities, the one ranked lower requires more cumulative mass to reach and thus receives a higher non-conformity score. In this work, we use the probability-based score throughout; comparison of the two scores is left to future work.

\section{Experimental Setup Details}
\label{app:setup}

\paragraph{Domain descriptions.}
The eight MMLU-Pro domains are chosen to span a range of reasoning types.
\emph{Mathematics} ($n{=}1{,}351$) tests formal reasoning and calculation.
\emph{Physics} ($n{=}1{,}299$) requires scientific reasoning with quantitative analysis.
\emph{Chemistry} ($n{=}1{,}132$) involves domain-specific knowledge with procedural reasoning.
\emph{Law} ($n{=}1{,}101$) tests interpretive reasoning over complex legal rules.
\emph{Engineering} ($n{=}969$) combines applied problem-solving across multiple disciplines.
\emph{Economics} ($n{=}844$) requires analytical reasoning over market and policy concepts.
\emph{Health} ($n{=}818$) covers biomedical knowledge with clinical reasoning.
\emph{Psychology} ($n{=}798$) involves behavioral science with theoretical reasoning.

\paragraph{Additional benchmarks and splits.}
Beyond MMLU-Pro we evaluate the primary ensemble on LogiQA \citep{liu2020logiqa} (logical reasoning; $n{=}8{,}678$, split 4{,}339 calibration / 4{,}339 test) and GPQA Extended \citep{rein2023gpqa} (graduate-level science; $n{=}546$, split 273 / 273).
Both use four answer options, whereas MMLU-Pro uses ten.
Each benchmark receives its own conformal calibration split and its own round-specific threshold: ``the same procedure'' means that the pipeline and target $\alpha$ are unchanged, \emph{not} that a threshold estimated on MMLU-Pro is transferred to another benchmark.
MMLU-Pro is calibrated separately within each domain; its aggregate counts are 8{,}312 samples split into 4{,}154 calibration and 4{,}158 test.
Results are in Appendix~\ref{app:additional_benchmarks}.

\paragraph{Secondary agent configuration.}
The generality experiment of Appendix~\ref{app:secondary_ensemble} replaces Qwen-3~32B with Amazon Nova Pro \citep{amazon2024nova}, giving the ensemble Claude Haiku~4.5 + Amazon Nova Pro + DeepSeek-R1.
It covers five MMLU-Pro domains (Engineering, Law, Chemistry, Physics, Mathematics) and runs $T{=}5$ debate rounds (indexed 0--4) rather than four; all other settings, including temperature, prompting, and the 50/50 split protocol, are unchanged.

\paragraph{Baseline descriptions.}
\emph{Single Agent (Greedy)} uses standard top-1 prediction from each individual LLM, providing a lower bound on ensemble performance.
\emph{Single Agent Self-Reflection} has a single agent iteratively review and revise its answer over $k$ rounds, testing whether multi-round reasoning without diversity suffices.
\emph{Majority Voting} has each agent cast a vote for its top-1 answer; the plurality winner is selected \citep{wang2023selfconsistency}.
\emph{Debate + Majority Voting} runs multi-agent debate for $T$ rounds, with final answers aggregated by majority vote \citep{du2023improvingfactualityreasoninglanguage}.

\paragraph{Agent configuration.}
The three models were selected for architectural and training diversity.
Claude Haiku~4.5 \citep{anthropic2025claudehaiku45} is a compact model from the Claude family.
DeepSeek-R1 \citep{deepseek2025r1} is a reasoning-specialized model trained with reinforcement learning.
Qwen-3~32B \citep{qwen2025qwen3} is an open-weight model from a distinct training pipeline.
All models are accessed via AWS Bedrock.

\paragraph{Debate configuration.}
We run $T{=}4$ rounds for all multi-agent methods.
At each round, agents receive the full summary of the previous round's responses.
The generation temperature is set to 0.7 with top-$p$ sampling at 1.0 and a maximum token budget of 4,096 per response.
A structured system prompt enforces the output format, requiring step-by-step reasoning within \texttt{<reasoning>} tags and a probability distribution within \texttt{<answer>} tags.

\paragraph{Probability parsing.}
Verbalized probabilities are extracted from agent responses using regex-based parsing that prioritizes content within \texttt{<answer>} tags.
The parser handles various formatting artifacts (e.g., markdown bold markers, escape characters).
Extracted values are clipped to $[0, 1]$ and renormalized to sum to one.
When parsing fails entirely for an agent, a uniform distribution is substituted (see Section~\ref{sec:debate}).
Across the primary MMLU-Pro runs (all 8{,}312 samples of calibration + test $\times$ 3 agents $\times$ 4 rounds = 99{,}744 agent-round responses), the overall parse failure rate (fallback to uniform) is only 0.77\%, indicating that the structured output format is parsed successfully in nearly all cases (Appendix~\ref{sec:parsing_failures}).
This figure is scoped to those runs and does not aggregate the additional benchmarks or the secondary ensemble.

\paragraph{Conformal calibration.}
For each domain, a 50/50 random split of the data produces calibration and test sets (exchangeability is maintained by random shuffling).
We evaluate at miscoverage levels $\alpha \in \{0.05, 0.10\}$, corresponding to target coverage rates of 95\% and 90\%.
The conformal threshold $\qhat$ is computed independently for each domain and each debate round, enabling analysis of how coverage and set size evolve across rounds.

\section{Point-Estimate Baselines}
\label{sec:baseline_results}

Table~\ref{tab:baselines} contextualizes our framework against standard point-estimate methods.
Debate + majority voting shows steady accuracy gains across rounds, from 60.8--92.2\% at round~0 (static majority vote) to 66.7--94.2\% by round~3, outperforming greedy decoding and self-reflection.
Most improvement occurs in round~1, with diminishing returns thereafter.
However, point-estimate methods provide no mechanism to flag uncertain or incorrect predictions.
The conformal approach of \S\ref{sec:conformal_results} addresses this gap: among cases resolved before the final round, singleton predictions achieve 90.0--96.8\% accuracy while non-singletons are escalated, yielding a set-valued act-or-escalate pipeline that point estimates do not provide.

\begin{table*}[t]
\centering
\small
\setlength{\tabcolsep}{4pt}
\begin{tabular}{@{}llcccccccc@{}}
\toprule
\textbf{Method} & \textbf{Model(s)} & \textbf{Engineering} & \textbf{Law} & \textbf{Chemistry} & \textbf{Physics} & \textbf{Math} & \textbf{Economics} & \textbf{Health} & \textbf{Psychology} \\
\midrule
\multirow{3}{*}{Greedy}
  & Claude Haiku    & 70.1 & 60.5 & 83.2 & 84.8 & 90.2 & 83.8 & 75.9 & 80.8 \\
  & DeepSeek-R1     & 58.4 & 63.9 & 77.2 & 78.4 & 88.4 & 82.8 & 76.8 & 79.7 \\
  & Qwen-3 32B     & 50.5 & 39.4 & 62.4 & 62.0 & 66.8 & 76.4 & 68.8 & 70.3 \\
\midrule
\multirow{3}{*}{\shortstack[l]{Self-Reflection\\(Rd 3)}}
  & Claude Haiku    & 72.7 & 59.5 & 86.6 & 86.2 & 92.1 & 85.2 & 76.0 & 80.6 \\
  & DeepSeek-R1     & \underline{78.3} & \underline{66.3} & \textbf{88.7} & \underline{87.4} & \underline{93.0} & 85.6 & \underline{80.1} & \underline{82.1} \\
  & Qwen-3 32B     & 56.8 & 39.4 & 69.5 & 72.2 & 78.5 & 78.1 & 68.6 & 72.6 \\
\midrule
\multirow{4}{*}{\shortstack[l]{Debate +\\Maj.\ Vote}}
  & H+D+Q (Rd 0) & 67.6 & 60.8 & 84.0 & 84.0 & 92.2 & \underline{85.7} & 79.3 & 81.5 \\
  & H+D+Q (Rd 1) & 77.9 & 64.8 & 88.2 & 89.3 & 93.9 & 87.0 & 80.1 & 81.5 \\
  & H+D+Q (Rd 2) & 79.6 & 65.8 & 88.7 & 89.8 & 94.0 & 87.2 & 80.3 & \textbf{82.3} \\
  & H+D+Q (Rd 3) & \textbf{80.0} & \textbf{66.7} & \underline{88.6} & \textbf{89.9} & \textbf{94.2} & \textbf{87.4} & \textbf{80.7} & \textbf{82.3} \\
\bottomrule
\end{tabular}
\caption{Point-estimate baselines (\%) on MMLU-Pro. \textbf{H}=Claude Haiku, \textbf{D}=DeepSeek-R1, \textbf{Q}=Qwen-3 32B. Self-reflection and debate report final-round (round~3) accuracy. Rd~0 corresponds to static majority voting (no debate). Debate + majority voting is the strongest point-estimate method, but point estimates provide no mechanism to flag incorrect predictions; the set-valued conformal approach of Table~\ref{tab:conformal} addresses this gap.}
\label{tab:baselines}
\end{table*}

\section{Consensus-Based Early Stopping Details}
\label{app:consensus}

Table~\ref{tab:consensus_detail} provides per-round statistics for consensus-based early stopping, showing how samples are distributed across stopping rounds.

\begin{table}[h]
\centering
\small
\begin{tabular}{@{}lcccc@{}}
\toprule
\textbf{Domain} & \textbf{R0} & \textbf{R1} & \textbf{R2} & \textbf{R3} \\
\midrule
Engineering & 36.3\% & 40.9\% & 14.3\% & 8.5\% \\
Law         & 34.2\% & 38.3\% & 14.7\% & 12.8\% \\
Chemistry   & 53.2\% & 35.9\% & 7.2\% & 3.7\% \\
Physics     & 52.3\% & 35.2\% & 8.2\% & 4.4\% \\
Math        & 60.8\% & 31.8\% & 4.8\% & 2.6\% \\
Economics   & 69.2\% & 21.7\% & 5.5\% & 3.7\% \\
Health      & 63.1\% & 23.7\% & 8.4\% & 4.8\% \\
Psychology  & 69.2\% & 20.2\% & 6.9\% & 3.8\% \\
\bottomrule
\end{tabular}
\caption{Distribution of samples across stopping rounds for consensus-based early stopping. The majority of samples stop at rounds 0--1, with domain difficulty reflected in the fraction reaching the final round.}
\label{tab:consensus_detail}
\end{table}

\section{Ablation: Uniform vs.\ Entropy-Based Weighting}
\label{app:weighting}

\begin{table*}[t]
\centering
\scriptsize
\setlength{\tabcolsep}{3.5pt}
\begin{tabular}{@{}cl cc cc cc cc cc cc cc cc@{}}
\toprule
& & \multicolumn{2}{c}{\textbf{Engineering}} & \multicolumn{2}{c}{\textbf{Law}} & \multicolumn{2}{c}{\textbf{Chemistry}} & \multicolumn{2}{c}{\textbf{Physics}} & \multicolumn{2}{c}{\textbf{Math}} & \multicolumn{2}{c}{\textbf{Economics}} & \multicolumn{2}{c}{\textbf{Health}} & \multicolumn{2}{c}{\textbf{Psychology}} \\
\cmidrule(lr){3-4} \cmidrule(lr){5-6} \cmidrule(lr){7-8} \cmidrule(lr){9-10} \cmidrule(lr){11-12} \cmidrule(lr){13-14} \cmidrule(lr){15-16} \cmidrule(lr){17-18}
\textbf{Metric} & \textbf{Rd} & \textbf{.05} & \textbf{.10} & \textbf{.05} & \textbf{.10} & \textbf{.05} & \textbf{.10} & \textbf{.05} & \textbf{.10} & \textbf{.05} & \textbf{.10} & \textbf{.05} & \textbf{.10} & \textbf{.05} & \textbf{.10} & \textbf{.05} & \textbf{.10} \\
\midrule
\multirow{4}{*}{$\qhat_t$}
  & 0 & .992 & .979 & .985 & .960 & .986 & .850 & .973 & .815 & .785 & .507 & .957 & .796 & .980 & .917 & .962 & .911 \\
  & 1 & .989 & .970 & .992 & .970 & .986 & .696 & .943 & .641 & .769 & .343 & .980 & .793 & .987 & .928 & .971 & .938 \\
  & 2 & .992 & .967 & .995 & .978 & .990 & .663 & .962 & .488 & .703 & .178 & .983 & .824 & .990 & .935 & .980 & .944 \\
  & 3 & .995 & .980 & .997 & .984 & .990 & .734 & .968 & .535 & .801 & .111 & .984 & .852 & .992 & .940 & .983 & .950 \\
\midrule
\multirow{4}{*}{\shortstack[l]{Fixed-rd.\\Cov. (\%)}}
  & 0 & 95.3 & 86.6 & 96.2 & 90.4 & 95.2 & 89.2 & 94.8 & 88.8 & 93.9 & 89.2 & 96.0 & 90.8 & 97.3 & 90.5 & 94.7 & 89.0 \\
  & 1 & 94.6 & 90.3 & 96.5 & 88.4 & 94.9 & 87.5 & 93.8 & 88.0 & 94.8 & 90.5 & 95.7 & 88.4 & 96.3 & 91.0 & 93.5 & 88.5 \\
  & 2 & 94.8 & 89.3 & 95.8 & 90.0 & 95.6 & 86.8 & 94.2 & 87.4 & 94.2 & 91.0 & 94.5 & 88.2 & 96.3 & 89.7 & 93.0 & 88.5 \\
  & 3 & 94.6 & 90.7 & 95.8 & 89.8 & 95.0 & 86.8 & 93.5 & 87.7 & 94.4 & 91.1 & 94.1 & 88.2 & 96.1 & 88.8 & 93.0 & 88.7 \\
\midrule
\multirow{4}{*}{$|\calC_t|$}
  & 0 & 5.02 & 3.04 & 6.48 & 4.03 & 2.82 & 1.40 & 2.06 & 1.30 & 1.22 & 0.96 & 1.97 & 1.22 & 3.92 & 1.71 & 2.58 & 1.59 \\
  & 1 & 3.77 & 2.41 & 6.40 & 3.58 & 1.87 & 1.04 & 1.28 & 1.02 & 1.06 & 0.98 & 1.93 & 1.09 & 3.68 & 1.54 & 2.26 & 1.55 \\
  & 2 & 3.81 & 1.98 & 6.59 & 3.62 & 1.74 & 1.01 & 1.28 & 1.00 & 1.02 & 0.98 & 1.79 & 1.07 & 4.16 & 1.44 & 2.25 & 1.46 \\
  & 3 & 3.81 & 2.13 & 6.54 & 3.83 & 1.78 & 1.01 & 1.25 & 1.00 & 1.02 & 0.99 & 1.75 & 1.07 & 4.23 & 1.37 & 2.33 & 1.49 \\
\midrule
\multirow{4}{*}{\shortstack[l]{Adaptive\\Sing. Rate}}
  & 0 & 20.6 & 31.8 & 1.1 & 7.1 & 36.2 & 65.7 & 47.4 & 71.4 & 78.4 & 95.9 & 49.8 & 79.4 & 17.6 & 50.9 & 40.6 & 59.4 \\
  & 1 & 44.9 & 56.5 & 4.0 & 15.4 & 69.1 & 95.4 & 81.1 & 97.4 & 94.4 & 97.6 & 61.1 & 91.2 & 24.4 & 61.4 & 48.6 & 67.7 \\
  & 2 & 49.5 & 65.2 & 6.5 & 19.2 & 76.0 & 98.6 & 85.2 & 99.7 & 97.9 & 98.2 & 67.3 & 93.4 & 28.6 & 67.5 & 53.1 & 72.2 \\
  & 3 & 52.0 & 67.6 & 7.4 & 21.2 & 78.1 & 99.3 & 87.4 & 99.8 & 98.4 & 98.5 & 70.1 & 94.1 & 31.1 & 72.9 & 54.6 & 73.7 \\
\midrule
\multirow{4}{*}{\shortstack[l]{Adaptive\\Sing. Acc.}}
  & 0 & 100.0 & 97.4 & 100.0 & 89.7 & 97.6 & 93.5 & 98.4 & 92.0 & 95.8 & 93.1 & 98.1 & 92.2 & 94.4 & 92.3 & 95.7 & 92.4 \\
  & 1 & 94.1 & 92.5 & 100.0 & 82.6 & 95.2 & 75.6 & 92.2 & 76.9 & 90.7 & 75.0 & 85.4 & 74.0 & 96.4 & 83.7 & 93.8 & 90.9 \\
  & 2 & 86.4 & 73.8 & 64.3 & 66.7 & 97.4 & 66.7 & 88.9 & 73.3 & 75.0 & 75.0 & 80.8 & 55.6 & 88.2 & 84.0 & 88.9 & 94.4 \\
  & 3 & 75.0 & 91.7 & 80.0 & 90.9 & 91.7 & 50.0 & 78.6 & 100.0 & 100.0 & 50.0 & 75.0 & 66.7 & 70.0 & 77.3 & 83.3 & 66.7 \\
\bottomrule
\end{tabular}
\caption{Conformal prediction results with \textbf{entropy-based weighting} ($\lambda{=}1$) and early stopping. Metric definitions match Table~\ref{tab:conformal}, but the scope is the eight MMLU-Pro domains only; this ablation was not run on LogiQA or GPQA. Compare against the corresponding MMLU-Pro columns of Table~\ref{tab:conformal} (uniform weighting). Differences are negligible: mean $|\Delta\text{Coverage}| = 0.6\%$, mean $|\Delta\text{Set Size}| = 0.06$, mean $|\Delta\text{Singleton Acc.}| = 0.4\%$ (see Table~\ref{tab:weighting_delta} for the per-domain breakdown).}
\label{tab:conformal_entropy}
\end{table*}

Our social probability aggregation (Eq.~\ref{eq:social_score}) admits different weighting strategies for combining agent probability distributions.
We compare two strategies:
\begin{itemize}[nosep]
    \item \textbf{Uniform}: $w_i = 1/N$ for all agents.
    \item \textbf{Entropy}: $w_i(x) = e^{-\lambda H(P_i)} / \sum_j e^{-\lambda H(P_j)}$, where $H(P_i) = -\sum_y P_i(y \mid x) \log P_i(y \mid x)$ is the Shannon entropy of agent $i$'s distribution, and $\lambda = 1$.
\end{itemize}
The entropy weighting upweights agents that are more confident (lower entropy) on each individual instance, in the spirit of epistemic social choice theory.

The MMLU-Pro columns of Table~\ref{tab:conformal} (uniform) and Table~\ref{tab:conformal_entropy} (entropy) present the full results; the ablation was run on MMLU-Pro only.
Table~\ref{tab:weighting_delta} summarizes the round~3 (final round) differences (entropy $-$ uniform) across all eight domains and both miscoverage levels.
The differences are small: the mean absolute change in coverage is 0.6\%, in set size 0.06, and in accuracy 0.4\%.
The largest set-size change is 0.37, and most domain--$\alpha$ combinations show coverage and accuracy shifts well under 1 percentage point.

\begin{table}[h]
\centering
\apptabwide
\begin{tabular}{@{}ll rrr@{}}
\toprule
\textbf{Domain} & $\boldsymbol{\alpha}$ & $\Delta$\textbf{Cov.} & $\Delta|\calC_t|$ & $\Delta$\textbf{Acc.} \\
\midrule
Engineering  & .05 & $-$0.4 & $-$0.03 & $-$0.2 \\
             & .10 & +0.2 & +0.00 & $-$0.2 \\
Law          & .05 & $-$0.5 & $-$0.37 & $-$0.7 \\
             & .10 & $-$0.9 & $-$0.11 & $-$0.7 \\
Chemistry    & .05 & $-$0.5 & $-$0.03 & +0.3 \\
             & .10 & $-$1.1 & +0.00 & $-$0.7 \\
Physics      & .05 & $-$0.6 & $-$0.02 & +0.0 \\
             & .10 & $-$0.6 & +0.00 & $-$0.6 \\
Math         & .05 & $-$0.2 & +0.01 & $-$0.1 \\
             & .10 & $-$2.1 & +0.01 & $-$2.5 \\
Economics    & .05 & +0.5 & +0.10 & +0.0 \\
             & .10 & +0.0 & +0.01 & +0.0 \\
Health       & .05 & +0.5 & +0.10 & +0.0 \\
             & .10 & $-$0.2 & $-$0.01 & $-$0.2 \\
Psychology   & .05 & $-$0.5 & $-$0.12 & +0.0 \\
             & .10 & +0.0 & +0.05 & +0.0 \\
\bottomrule
\end{tabular}
\caption{Final-round differences between entropy and uniform weighting (entropy $-$ uniform) in coverage (\%), average set size, and singleton accuracy (\%). Differences are small: mean $|\Delta\text{Coverage}| = 0.6\%$, mean $|\Delta\text{Set Size}| = 0.06$, mean $|\Delta\text{Singleton Acc.}| = 0.4\%$, max $|\Delta\text{Coverage}| = 2.1\%$. The results are relatively insensitive to the tested weighting strategy.}
\label{tab:weighting_delta}
\end{table}

This result suggests that the reported outcomes are relatively insensitive to this tested entropy-based weighting rule; the ablation does not directly measure whether agent distributions themselves converge, only that the downstream metrics are similar under the two aggregation rules.
It is consistent with findings in the multi-agent debate literature, where deliberation tends to produce outputs that are robust to the choice of aggregation rule \citep{du2023improvingfactualityreasoninglanguage}.

\section{Stringent Coverage: $\alpha = 0.01$ Results}
\label{app:alpha001}

Table~\ref{tab:conformal_alpha001} presents detailed conformal prediction results at $\alpha = 0.01$ (target coverage: 99\%) across all eight MMLU-Pro domains.
Unlike the $\alpha = 0.05$ and $\alpha = 0.10$ settings reported in Table~\ref{tab:conformal}, this stringent confidence level reveals a qualitatively different behavior: \emph{prediction set sizes increase across debate rounds}, often approaching the full option set of 10.

\begin{table*}[t]
\centering
\scriptsize
\setlength{\tabcolsep}{3.5pt}
\begin{tabular}{@{}cl cccccccc@{}}
\toprule
& & \textbf{Engineering} & \textbf{Law} & \textbf{Chemistry} & \textbf{Physics} & \textbf{Math} & \textbf{Economics} & \textbf{Health} & \textbf{Psychology} \\
\midrule
\multirow{4}{*}{$\qhat_t$}
  & 0 & .997 & .994 & .997 & .997 & .990 & .993 & .997 & .989 \\
  & 1 & .997 & 1.00 & .999 & 1.00 & .997 & .997 & 1.00 & .997 \\
  & 2 & 1.00 & 1.00 & 1.00 & 1.00 & .999 & 1.00 & 1.00 & 1.00 \\
  & 3 & 1.00 & 1.00 & 1.00 & 1.00 & 1.00 & 1.00 & 1.00 & 1.00 \\
\midrule
\multirow{4}{*}{\shortstack[l]{Fixed-rd.\\Cov. (\%)}}
  & 0 & 99.0 & 99.3 & 99.3 & 99.4 & 99.0 & 99.3 & 99.3 & 99.8 \\
  & 1 & 98.6 & 100.0 & 99.3 & 99.4 & 99.6 & 99.0 & 100.0 & 99.5 \\
  & 2 & 99.6 & 100.0 & 100.0 & 99.8 & 99.1 & 100.0 & 100.0 & 100.0 \\
  & 3 & 99.6 & 100.0 & 100.0 & 99.8 & 99.3 & 100.0 & 100.0 & 100.0 \\
\midrule
\multirow{4}{*}{$|\calC_t|$}
  & 0 & 7.89 & 8.59 & 7.08 & 7.50 & 3.25 & 5.43 & 8.02 & 5.97 \\
  & 1 & 6.23 & 9.20 & 6.01 & 6.25 & 3.00 & 5.37 & 8.84 & 7.25 \\
  & 2 & 7.97 & 9.20 & 8.26 & 8.28 & 2.58 & 7.31 & 8.84 & 8.66 \\
  & 3 & 7.97 & 9.20 & 8.26 & 8.28 & 4.03 & 7.31 & 8.84 & 8.66 \\
\midrule
\multirow{4}{*}{\shortstack[l]{Adaptive\\Sing. Rate}}
  & 0 & 3.9 & 0.0 & 4.4 & 4.0 & 46.2 & 16.1 & 2.4 & 7.8 \\
  & 1 & 19.6 & 0.0 & 17.7 & 18.3 & 55.8 & 21.1 & 2.4 & 10.0 \\
  & 2 & 19.6 & 0.0 & 17.7 & 18.3 & 65.4 & 21.1 & 2.4 & 10.0 \\
  & 3 & 19.6 & 0.0 & 17.7 & 18.3 & 65.4 & 21.1 & 2.4 & 10.0 \\
\midrule
\multirow{4}{*}{\shortstack[l]{Adaptive\\Sing. Acc.}}
  & 0 & 100.0 & --- & 100.0 & 100.0 & 99.7 & 100.0 & 100.0 & 100.0 \\
  & 1 & 97.4 & --- & 100.0 & 98.9 & 100.0 & 100.0 & --- & 100.0 \\
  & 2 & --- & --- & --- & --- & 93.8 & --- & --- & --- \\
  & 3 & --- & --- & --- & --- & --- & --- & --- & --- \\
\midrule
\multirow{4}{*}{\shortstack[l]{Accuracy\\(\%)}}
  & 0 & 65.77 & 60.62 & 82.86 & 83.54 & 91.72 & 85.78 & 80.44 & 81.45 \\
  & 1 & 77.32 & 64.07 & 86.57 & 89.08 & 94.08 & 86.97 & 81.91 & 81.20 \\
  & 2 & 79.79 & 64.97 & 87.46 & 89.08 & 93.93 & 86.49 & 82.15 & 81.70 \\
  & 3 & 79.79 & 66.06 & 87.28 & 89.54 & 94.23 & 86.73 & 82.89 & 81.70 \\
\bottomrule
\end{tabular}
\caption{Conformal prediction results at $\alpha = 0.01$ (99\% target coverage) across debate rounds~0--3 on MMLU-Pro. Unlike $\alpha \in \{0.05, 0.10\}$, prediction set sizes \emph{increase} with debate rounds due to \emph{calibration threshold saturation}: as debate sharpens confidence, the few remaining errors produce extreme non-conformity scores that push $\qhat \to 1.0$ at the 99th percentile, collapsing all labels into the prediction set. Thresholds are rounded for display; entries shown as $1.00$ may correspond to values strictly below~$1$, which is why set sizes do not always reach $|\calY|{=}10$ even when $\qhat$ displays as $1.00$.}
\label{tab:conformal_alpha001}
\end{table*}

\paragraph{Why do prediction sets grow at stringent $\alpha$?}
The key mechanism is \emph{calibration threshold saturation}.
The conformal threshold $\qhat$ is set as the $\lceil (1-\alpha)(n+1) \rceil / n$ quantile of the calibration nonconformity scores $s_i = 1 - P_{\text{social}}(y_i \mid x_i)$.
At $\alpha = 0.01$, this corresponds to approximately the 99th percentile, effectively the worst-case calibration example.

As debate progresses, agents become increasingly confident: the social choice distribution $P_{\text{social}}$ concentrates more mass on the majority answer.
For the majority of examples where the agents are \emph{correct}, this sharpening reduces nonconformity scores toward zero.
However, for the small fraction of examples where agents are \emph{confidently wrong}, the nonconformity score $s = 1 - P_{\text{social}}(y^* \mid x)$ approaches 1.0, since the true label $y^*$ receives near-zero probability.

At $\alpha = 0.05$ or $\alpha = 0.10$, the calibration threshold is determined by a less extreme quantile (95th or 90th), which is not dominated by these confidently-wrong tail cases.
But at $\alpha = 0.01$, the 99th-percentile quantile is precisely these worst-case examples, pushing $\qhat \to 1.0$.
When $\qhat = 1.0$, the prediction set $\calC(x) = \{y : s(x,y) \leq \qhat\} = \{y : 1 - P_{\text{social}}(y \mid x) \leq 1\}$ includes \emph{all} options, as every candidate trivially satisfies the inclusion criterion.

This explains the striking pattern in Table~\ref{tab:conformal_alpha001}: Math, which achieves the highest accuracy ($>$91\%), maintains small prediction sets ($|\calC| \approx 3$) through rounds~0--2 because even the 99th-percentile calibration example retains some probability on the correct answer. By round~3 its threshold rounds to 1.00 and the average set size rises to 4.03: an increase, but not yet the full collapse to all ten labels, since the underlying threshold remains strictly below~1.
In contrast, domains like Law (60--66\% accuracy) saturate earlier: their thresholds reach 1.00 by round~1 and the set size climbs to 9.20, close to the full option set, as the larger fraction of errors produces extreme nonconformity scores earlier in the debate.

This phenomenon points to a tension between debate-driven confidence sharpening and stringent coverage guarantees: as debate proceeds, consensus formation is accompanied by larger nonconformity scores on the remaining errors, which inflate the calibration threshold.
At the moderate $\alpha$ values studied here (0.05 and 0.10), this trade-off is more favorable: prediction sets remain substantially more informative while maintaining near-target coverage.
At very stringent $\alpha$, the tail behavior dominates, rendering the prediction sets uninformative.
These results suggest that $\alpha \in \{0.05, 0.10\}$ provides useful operating points in our experiments.

\section{Statistical Variability: Repeated Random-Split Replication}
\label{app:repeated_splits}

The coverage values in Table~\ref{tab:conformal} are computed from a single random 50/50 calibration--test split per domain or benchmark.
To quantify the variability introduced by this split, we repeat the full conformal calibration and evaluation pipeline $K{=}20$ times with independent random permutations of each domain's data, reporting the mean coverage and standard deviation across splits.
This replication covers the eight MMLU-Pro domains only; it does not include the additional benchmarks of Appendix~\ref{app:additional_generality}.
Table~\ref{tab:coverage_ci} presents the results.

At $\alpha{=}0.05$ (target: 95\%), mean coverage across all eight domains and four rounds ranges from 94.4\% to 95.6\%, with standard deviations of 1.0--1.7pp (e.g., Engineering round~3: $95.2 \pm 1.6$; Law round~3: $95.0 \pm 1.0$).
At $\alpha{=}0.10$ (target: 90\%), mean coverage ranges from 88.7\% to 91.0\%, with standard deviations of 1.1--2.4pp.
The single-run values reported in Table~\ref{tab:conformal} are consistent with the variability observed across these repeated random splits.
Smaller domains exhibit larger variability (Economics: $\pm$2.3pp at $\alpha{=}0.10$, $n{=}422$) than larger ones (Chemistry: $\pm$1.3pp, $n{=}566$), as expected from the $O(1/\sqrt{n})$ scaling of split conformal prediction.

\begin{table*}[t]
\centering
\scriptsize
\setlength{\tabcolsep}{3pt}
\begin{tabular}{@{}cl cccccccc@{}}
\toprule
& & \textbf{Eng.} & \textbf{Law} & \textbf{Chem.} & \textbf{Phys.} & \textbf{Math} & \textbf{Econ.} & \textbf{Health} & \textbf{Psych.} \\
\midrule
\multicolumn{10}{@{}l}{\textbf{$\alpha{=}0.05$ (target: 95\%)}} \\
\midrule
\multirow{4}{*}{\shortstack[l]{Fixed-rd.\\Cov. (\%)}}
  & 0 & 95.2{\scriptsize$\pm$1.2} & 95.4{\scriptsize$\pm$1.3} & 94.9{\scriptsize$\pm$1.4} & 94.8{\scriptsize$\pm$1.3} & 95.0{\scriptsize$\pm$1.5} & 94.4{\scriptsize$\pm$1.6} & 95.5{\scriptsize$\pm$1.1} & 95.1{\scriptsize$\pm$1.5} \\
  & 1 & 95.3{\scriptsize$\pm$1.4} & 95.3{\scriptsize$\pm$1.2} & 94.8{\scriptsize$\pm$1.2} & 94.6{\scriptsize$\pm$1.3} & 94.9{\scriptsize$\pm$1.4} & 94.5{\scriptsize$\pm$1.3} & 95.3{\scriptsize$\pm$1.4} & 95.4{\scriptsize$\pm$1.5} \\
  & 2 & 94.8{\scriptsize$\pm$1.7} & 95.1{\scriptsize$\pm$1.2} & 94.9{\scriptsize$\pm$1.1} & 94.7{\scriptsize$\pm$1.2} & 95.0{\scriptsize$\pm$1.4} & 94.5{\scriptsize$\pm$1.4} & 95.5{\scriptsize$\pm$1.4} & 95.5{\scriptsize$\pm$1.4} \\
  & 3 & 95.2{\scriptsize$\pm$1.6} & 95.0{\scriptsize$\pm$1.0} & 95.0{\scriptsize$\pm$1.0} & 94.9{\scriptsize$\pm$1.0} & 95.0{\scriptsize$\pm$1.4} & 94.7{\scriptsize$\pm$1.7} & 95.6{\scriptsize$\pm$1.1} & 95.6{\scriptsize$\pm$1.7} \\
\midrule
\multicolumn{10}{@{}l}{\textbf{$\alpha{=}0.10$ (target: 90\%)}} \\
\midrule
\multirow{4}{*}{\shortstack[l]{Fixed-rd.\\Cov. (\%)}}
  & 0 & 90.9{\scriptsize$\pm$1.8} & 90.0{\scriptsize$\pm$1.9} & 90.1{\scriptsize$\pm$1.1} & 89.5{\scriptsize$\pm$1.4} & 90.1{\scriptsize$\pm$1.9} & 89.0{\scriptsize$\pm$2.2} & 90.1{\scriptsize$\pm$2.1} & 90.4{\scriptsize$\pm$2.4} \\
  & 1 & 90.9{\scriptsize$\pm$1.8} & 89.6{\scriptsize$\pm$2.2} & 90.1{\scriptsize$\pm$1.3} & 89.6{\scriptsize$\pm$1.3} & 89.6{\scriptsize$\pm$1.7} & 88.8{\scriptsize$\pm$2.3} & 90.2{\scriptsize$\pm$2.2} & 91.0{\scriptsize$\pm$2.3} \\
  & 2 & 90.7{\scriptsize$\pm$2.0} & 89.6{\scriptsize$\pm$1.9} & 90.2{\scriptsize$\pm$1.3} & 89.9{\scriptsize$\pm$1.4} & 89.7{\scriptsize$\pm$1.9} & 88.8{\scriptsize$\pm$2.2} & 90.3{\scriptsize$\pm$1.9} & 90.6{\scriptsize$\pm$2.2} \\
  & 3 & 90.7{\scriptsize$\pm$2.0} & 90.0{\scriptsize$\pm$1.6} & 90.3{\scriptsize$\pm$1.3} & 89.9{\scriptsize$\pm$1.4} & 89.7{\scriptsize$\pm$1.8} & 88.7{\scriptsize$\pm$2.3} & 90.2{\scriptsize$\pm$1.7} & 90.4{\scriptsize$\pm$2.3} \\
\bottomrule
\end{tabular}
\caption{Coverage (\%) reported as mean $\pm$ standard deviation over $K{=}20$ repeated random calibration--test splits of the eight MMLU-Pro domains. At $\alpha{=}0.05$, mean coverage is 94.4--95.6\% with std $\leq$1.7pp across all domain--round combinations; at $\alpha{=}0.10$, mean coverage is 88.7--91.0\% with std $\leq$2.4pp. The single-run values from Table~\ref{tab:conformal} are consistent with the variability observed across these repeated random splits.}
\label{tab:coverage_ci}
\end{table*}

\section{Wrong-Consensus Convergence and Conformal Safety Analysis}
\label{app:sycophancy}

Multi-agent debate improves accuracy through deliberation, but it also carries a systemic risk: agents may converge to unanimous agreement on the \emph{wrong} answer through social reinforcement, a pattern consistent with known sycophantic tendencies in LLMs.
We provide a detailed analysis of this phenomenon and quantify how conformal prediction mitigates (or, in rare cases, compounds) the risk.

\subsection{Wrong-Consensus Convergence Across Rounds}

Table~\ref{tab:sycophancy_detail} reports, for each domain, the number of inference samples where agents initially disagree (no unanimous consensus at round~0) but later converge to unanimous wrong consensus.
We use the inference split (second half of each domain's dataset) to ensure alignment with the conformal prediction evaluation.

\begin{table}[t]
\centering
\scriptsize
\setlength{\tabcolsep}{3pt}
\begin{tabular}{@{}l r rrr@{}}
\toprule
& & \multicolumn{3}{c}{\textbf{$\to$ Wrong Consensus (\%)}} \\
\cmidrule(lr){3-5}
\textbf{Domain} & \textbf{Init.\ Disagr.} & \textbf{Rd 1} & \textbf{Rd 2} & \textbf{Rd 3} \\
\midrule
Engineering & 312 & 37 (11.9) & 65 (20.8) & 75 (24.0) \\
Law         & 388 & 90 (23.2) & 125 (32.2) & 135 (34.8) \\
Chemistry   & 288 & 29 (10.1) & 45 (15.6) & 55 (19.1) \\
Physics     & 313 & 31 (9.9) & 44 (14.1) & 50 (16.0) \\
Math        & 256 & 19 (7.4) & 30 (11.7) & 31 (12.1) \\
Economics   & 135 & 24 (17.8) & 32 (23.7) & 35 (25.9) \\
Health      & 145 & 24 (16.6) & 36 (24.8) & 46 (31.7) \\
Psychology  & 126 & 24 (19.0) & 37 (29.4) & 42 (33.3) \\
\midrule
\textbf{All} & 1{,}963 & 278 (14.2) & 414 (21.1) & 469 (23.9) \\
\bottomrule
\end{tabular}
\caption{Wrong-consensus convergence among initially disagreeing cases. \textbf{Init.\ Disagr.}~=~inference samples without unanimous consensus at round~0. Columns show the count (and rate) of samples converging to unanimous \emph{wrong} consensus by each round. By round~3, 23.9\% of initially disputed questions have reached wrong consensus.}
\label{tab:sycophancy_detail}
\end{table}

The cumulative fraction of initially disputed cases that reach wrong consensus increases across rounds, from 14.2\% at round~1 to 23.9\% at round~3.
The rate varies substantially by domain: Law (34.8\%) and Psychology (33.3\%) are highest, while Math (12.1\%) is lowest.

\subsection{Wrong-Consensus Rejection by Conformal Prediction}

We next ask: when agents reach unanimous wrong consensus, does conformal prediction catch the error?
For each round, we identify all inference samples with unanimous wrong consensus and check whether the conformal prediction set has $|\calC| > 1$ (correctly flagging the error for human review) or $|\calC| = 1$ (letting the wrong answer through as a singleton).


Table~\ref{tab:sycophancy} shows that at $\alpha{=}0.05$, conformal prediction correctly flags \textbf{81.9\%} of wrong-consensus cases across the eight MMLU-Pro domains.
The rejection rate is highest in domains where prediction sets tend to be large: Law achieves 97.5\% rejection, consistent with its larger prediction sets even when agents agree.
In contrast, Math (11.4\%) and Physics (70.7\%) show the lowest rejection rates: in these high-accuracy domains, agents' probability distributions are so concentrated that even wrong answers produce near-singleton sets.

\subsection{Conformal-Introduced Wrong Singletons}

A complementary failure mode occurs when the conformal set becomes a wrong singleton even though the agents are not in unanimous wrong consensus.
These cases lie outside the final-round wrong-consensus set analyzed above and are therefore not counted among the 586 wrong-consensus errors, but they represent a distinct route by which the conformal layer can commit to a wrong automated action.

\begin{table}[t]
\centering
\scriptsize
\setlength{\tabcolsep}{2.5pt}
\begin{tabular}{@{}l rrrr rrrr@{}}
\toprule
& \multicolumn{4}{c}{$\alpha = 0.05$} & \multicolumn{4}{c}{$\alpha = 0.10$} \\
\cmidrule(lr){2-5} \cmidrule(lr){6-9}
\textbf{Domain} & \textbf{Rd 0} & \textbf{Rd 1} & \textbf{Rd 2} & \textbf{Rd 3} & \textbf{Rd 0} & \textbf{Rd 1} & \textbf{Rd 2} & \textbf{Rd 3} \\
\midrule
Engineering & 0 & 0 & 0 & 0 & 0 & 0 & 0 & 0 \\
Law         & 0 & 0 & 0 & 0 & 1 & 0 & 0 & 0 \\
Chemistry   & 1 & 0 & 0 & 0 & 14 & 19 & 8 & 2 \\
Physics     & 0 & 0 & 0 & 0 & 24 & 19 & 12 & 6 \\
Math        & 17 & 9 & 2 & 2 & 31 & 3 & 0 & 0 \\
Economics   & 0 & 0 & 0 & 0 & 11 & 3 & 1 & 0 \\
Health      & 0 & 0 & 0 & 0 & 2 & 0 & 0 & 0 \\
Psychology  & 0 & 0 & 0 & 0 & 2 & 0 & 0 & 0 \\
\midrule
\textbf{All} & 18 & 9 & 2 & 2 & 85 & 44 & 21 & 8 \\
\textbf{\% of $n$} & 0.4 & 0.2 & 0.0 & 0.0 & 2.0 & 1.1 & 0.5 & 0.2 \\
\bottomrule
\end{tabular}
\caption{New wrong singletons introduced by conformal prediction \emph{among cases without unanimous wrong consensus}: cases where agents do not all agree on the wrong answer, but the conformal set collapses to $|\calC|{=}1$ containing the wrong answer. Counts per round; \textbf{\% of $n$} relative to total inference samples ($n{=}4{,}158$). At $\alpha{=}0.05$ round~3, conformal flags 480 of the 586 final-round wrong-consensus cases while producing only 2 such new wrong singletons (a 240:1 ratio). At $\alpha{=}0.10$ the introduced count rises to 8.}
\label{tab:introduced_wrong}
\end{table}

Table~\ref{tab:introduced_wrong} reveals that this risk is minimal.
At $\alpha{=}0.05$ and round~3, only \textbf{2 out of 4{,}158} inference samples (0.05\%) have conformal-introduced wrong singletons.
At $\alpha{=}0.10$, the rate is slightly higher at 8 samples (0.2\%), concentrated in Chemistry and Physics where $\alpha{=}0.10$ produces very small sets.
Two observations:
\begin{enumerate}[nosep]
    \item The count generally decreases across rounds, from 18 at round~0 to 2 at round~3 for $\alpha{=}0.05$.
    \item The count is higher at $\alpha{=}0.10$ than $\alpha{=}0.05$, which is consistent with tighter sets being more likely to collapse to a wrong singleton.
\end{enumerate}

\subsection{Final-Round Error-Flagging Balance}

The final-round balance between wrong-consensus cases flagged and newly introduced wrong singletons is strongly favorable.
At $\alpha{=}0.05$ and round~3, conformal prediction:
\begin{itemize}[nosep]
    \item \textbf{Flags} 480 of the 586 final-round wrong-consensus cases with $|\calC_3|>1$ (81.9\%);
    \item \textbf{Produces} only 2 new wrong singletons outside the final-round wrong-consensus set.
\end{itemize}
The ratio of final-round wrong-consensus cases flagged to newly introduced wrong singletons is approximately \textbf{240:1}.
At $\alpha{=}0.10$ the number of introduced wrong singletons rises to 8 (0.2\% of inference samples, Table~\ref{tab:introduced_wrong}), still far below the number of wrong-consensus cases flagged at that level.
This balance indicates that at the final round, calibrated refusal flags far more wrong-consensus cases than it converts into new wrong singletons, with a small downside risk that grows as $\alpha$ increases.

\section{Scalar Abstention Baseline Comparison}
\label{app:abstention_baselines}

A natural question is whether the conformal guarantee does real work, or whether any conservative abstention rule catches roughly the same wrong-consensus cases.
We therefore compare conformal prediction against three \emph{scalar} abstention rules that operate on the same social probability vector.
Define the top answer, its probability, and the distribution entropy as
\begin{align*}
\widehat y(x) &= \arg\max_y \psocial(y \mid x), \\
m(x) &= \max_y \psocial(y \mid x), \\
H(x) &= -\textstyle\sum_y \psocial(y \mid x) \log \psocial(y \mid x).
\end{align*}
Each scalar baseline implements a binary \emph{gate}: \textsc{act} outputs the single answer $\widehat y(x)$, whereas \textsc{defer} returns no answer and routes the case for review.
The three gates are:
\begin{itemize}[nosep,leftmargin=*]
    \item \textbf{Fixed $\tau{=}0.95$:} act iff $m(x) \geq 0.95$; a hard-coded threshold requiring no calibration data.
    \item \textbf{Top-1 (cal-tuned):} act iff $m(x) \geq \tau^{\mathrm{top}}_{d,t}$.
    \item \textbf{Entropy (cal-tuned):} act iff $H(x) \leq \tau^{\mathrm{ent}}_{d,t}$; low entropy indicates a more concentrated social distribution.
\end{itemize}
For each domain $d$ and round $t$, the calibrated thresholds are selected from all distinct calibration-set score values.
Among thresholds attaining at least 95\% acted-case accuracy on the calibration half, we choose the one that acts on the most examples; if none attains the target, we use the most conservative threshold that still acts on at least one calibration example.
This 95\% target is an empirical acted-accuracy tuning criterion, not a coverage guarantee.
All methods use the same multi-round protocol as conformal: a sample stops at the first round where it passes the method's gate.
We omit a fourth rule, majority agreement (\emph{Vote $k{=}2$}): unanimous consensus trivially satisfies a 2-of-3 agreement condition, so it intercepts 0\% of \emph{all} consensus cases and is uninformative here.

\paragraph{Difference from conformal prediction.}
The distinction concerns output semantics, not access to information.
The scalar baselines deliberately reduce the full probability vector to $m(x)$ or $H(x)$ solely to gate the top answer, so a deferral says only that the scalar threshold was not met: it does not identify which alternative answers remain plausible.
Conformal prediction instead thresholds \emph{every} label,
\[
\calC(x) = \{y : \psocial(y \mid x) \geq 1 - \qhat\},
\]
acts when $|\calC(x)| = 1$, and otherwise exposes the retained candidates for human escalation.
Its corrected conformal quantile (Eq.~\ref{eq:qhat}) also yields a finite-sample marginal coverage guarantee (Theorem~\ref{thm:coverage}), which the scalar gates do not provide: their thresholds are fixed by empirical sweep with no assurance of transfer to new data.

\paragraph{Wrong-consensus interception across methods.}
Table~\ref{tab:interception_comparison} extends the wrong-consensus analysis of Table~\ref{tab:combined_stopping} to the three scalar gates, evaluating each method's decision at the final debate round.
A consensus case is \emph{intercepted} if the method would defer on it at round~3 (the gate is not met, so the case is escalated rather than acted upon); conformal interception follows the main-text definition $|\calC| > 1$.
Rejection rates are computed over all 586 wrong-consensus and 3{,}434 correct-consensus cases in the inference split, matching Table~\ref{tab:combined_stopping}.
The aggregate comparison supports a deliberately limited claim: in this setting, scalar confidence rules match conformal wrong-consensus interception (conformal 81.9\%, Entropy 80.0\%, Top-1 79.7\%, Fixed $\tau$ 79.4\%).
The calibration-tuned gates achieve this at higher correct-consensus over-rejection (33.5\% and 35.1\% vs.\ conformal's 31.9\%), while Fixed $\tau$ over-rejects least (26.8\%) but is also the least adaptive.
The key distinctions are therefore not the aggregate interception rate but domain-level selectivity, the information returned upon escalation, and the guarantee.

\begin{table*}[t]
\centering
\small
\setlength{\tabcolsep}{2.5pt}
\begin{tabular}{@{}l r rrr cc cc cc cc@{}}
\toprule
& & \multicolumn{3}{c}{\textbf{WC Convergence}} & \multicolumn{2}{c}{\textbf{Conformal}} & \multicolumn{2}{c}{\textbf{Fixed $\tau{=}.95$}} & \multicolumn{2}{c}{\textbf{Top-1 (cal)}} & \multicolumn{2}{c}{\textbf{Entropy (cal)}} \\
\cmidrule(lr){3-5} \cmidrule(lr){6-7} \cmidrule(lr){8-9} \cmidrule(lr){10-11} \cmidrule(lr){12-13}
\textbf{Domain} & $n$ & \textbf{Disagr.} & \textbf{$\to$WC} & \textbf{WC\%} & \textbf{WC\,R} & \textbf{CC\,R} & \textbf{WC\,R} & \textbf{CC\,R} & \textbf{WC\,R} & \textbf{CC\,R} & \textbf{WC\,R} & \textbf{CC\,R} \\
\midrule
Engineering & 485 & 312 & 75 & 24.0 & 88.0 & 41.1 & 83.1 & 25.6 & 89.2 & 43.7 & 91.6 & 46.7 \\
Law         & 551 & 388 & 135 & 34.8 & 97.5 & 95.9 & 88.3 & 75.3 & 98.8 & 97.4 & 98.8 & 99.1 \\
Chemistry   & 566 & 288 & 55 & 19.1 & 82.5 & 18.8 & 77.8 & 10.5 & 66.7 & 5.9 & 61.9 & 5.3 \\
Physics     & 650 & 313 & 50 & 16.0 & 70.7 & 8.4 & 69.0 & 10.0 & 41.4 & 3.5 & 43.1 & 3.9 \\
Math        & 676 & 256 & 31 & 12.1 & 11.4 & 0.6 & 60.0 & 6.3 & 20.0 & 0.3 & 17.1 & 0.2 \\
Economics   & 422 & 135 & 35 & 25.9 & 62.0 & 26.5 & 52.0 & 25.7 & 66.0 & 37.6 & 66.0 & 37.8 \\
Health      & 409 & 145 & 46 & 31.7 & 88.1 & 69.3 & 82.1 & 53.4 & 95.5 & 86.6 & 100.0 & 97.0 \\
Psychology  & 399 & 126 & 42 & 33.3 & 91.2 & 43.3 & 91.2 & 44.9 & 92.6 & 53.6 & 92.6 & 55.7 \\
\midrule
\textbf{All} & \textbf{4{,}158} & \textbf{1{,}963} & \textbf{469} & \textbf{23.9} & \textbf{81.9} & \underline{31.9} & 79.4 & \textbf{26.8} & 79.7 & 33.5 & \underline{80.0} & 35.1 \\
\bottomrule
\end{tabular}
\caption{Wrong-consensus interception at round~3 ($\alpha{=}0.05$). \textbf{Disagr.}~=~samples without unanimous consensus at round~0. \textbf{$\to$WC}~=~converge to unanimous wrong consensus by round~3. \textbf{WC\,R}~=~wrong-consensus rejection rate (\%): the fraction of the 586 wrong-consensus cases on which the method would defer at round~3. \textbf{CC\,R}~=~correct-consensus over-rejection rate (\%) over the 3{,}434 correct-consensus cases. In the aggregate row, \textbf{bold} marks the best value and \underline{underline} the second best (highest WC\,R, lowest CC\,R). The conformal columns reproduce Table~\ref{tab:combined_stopping}. Scalar rules achieve comparable aggregate interception, while selectivity varies substantially by domain.}
\label{tab:interception_comparison}
\end{table*}

\paragraph{Domain case studies: Health and Chemistry.}
Health and Chemistry illustrate the interception--over-rejection tradeoff most clearly.
On \emph{Health} (67 wrong-consensus, 335 correct-consensus cases), the calibration-tuned gates achieve near-perfect wrong-consensus interception (Entropy: 100.0\%, Top-1: 95.5\%) but reject nearly all correct consensus as well (CC\,R = 97.0\% and 86.6\%); they are not selectively catching errors but deferring on almost everything.
Conformal strikes a middle ground, intercepting 88.1\% of wrong-consensus errors while over-rejecting 69.3\% of correct ones: 6pp more interception than Fixed $\tau$ at a proportionate increase in over-rejection.
On \emph{Chemistry} (63 wrong-consensus, 494 correct-consensus cases), conformal leads all methods on interception (82.5\%) while keeping over-rejection moderate (18.8\%); the calibration-tuned gates drop to 61.9--66.7\% interception for only marginally lower over-rejection (5--6\%), missing roughly a third of wrong-consensus errors to save a few correct ones: a poor tradeoff when the goal is safety.
Across both domains, conformal's set-valued mechanism yields a more balanced operating point than either the permissive Fixed $\tau$ or the aggressive calibration-tuned gates.

\section{Response Parsing Failure Analysis}
\label{sec:parsing_failures}

A small fraction of model responses failed to parse into a valid answer option.
Table~\ref{tab:parsing_failures} reports the complete failure rate for each model across all eight MMLU-Pro domains, aggregated over all debate rounds of the primary ensemble; it does not aggregate the additional benchmarks of Appendix~\ref{app:additional_benchmarks} or the secondary ensemble of Appendix~\ref{app:secondary_ensemble}.
Overall, 764 out of 99{,}744 total responses (0.77\%) could not be parsed.
Claude-Haiku exhibited the lowest failure rate (0.02\%), while DeepSeek-R1 showed the highest (2.05\%).
As described in Section~\ref{sec:debate}, these failures are handled by substituting a uniform distribution over the label set for the affected agent-round, so all samples remain in the calibration and evaluation pipeline. Because the failure rate is negligible (0.77\%), the uniform substitutions do not materially affect the reported coverage guarantees.

\begin{table*}[t]
\centering
\scriptsize
\setlength{\tabcolsep}{1pt}
\caption{Response parsing failure rates across models and domains (all rounds combined). Each cell shows failures\,/\,total responses (\%).}
\label{tab:parsing_failures}
\begin{tabular}{l rc rc rc rc rc rc rc rc r}
\toprule
 & \multicolumn{2}{c}{Law} & \multicolumn{2}{c}{Health} & \multicolumn{2}{c}{Economics} & \multicolumn{2}{c}{Psychology} & \multicolumn{2}{c}{Chemistry} & \multicolumn{2}{c}{Engineering} & \multicolumn{2}{c}{Math} & \multicolumn{2}{c}{Physics} & Overall \\
\cmidrule(lr){2-3}\cmidrule(lr){4-5}\cmidrule(lr){6-7}\cmidrule(lr){8-9}\cmidrule(lr){10-11}\cmidrule(lr){12-13}\cmidrule(lr){14-15}\cmidrule(lr){16-17}\cmidrule(lr){18-18}
Model & Fail & Rate & Fail & Rate & Fail & Rate & Fail & Rate & Fail & Rate & Fail & Rate & Fail & Rate & Fail & Rate & Rate \\
\midrule
Claude-Haiku & 1/4{,}404 & 0.02\% & 1/3{,}272 & 0.03\% & 0/3{,}376 & 0.00\% & 0/3{,}192 & 0.00\% & 2/4{,}528 & 0.04\% & 0/3{,}876 & 0.00\% & 2/5{,}404 & 0.04\% & 1/5{,}196 & 0.02\% & 0.02\% \\
DeepSeek-R1  & 128/4{,}404 & 2.91\% & 41/3{,}272 & 1.25\% & 54/3{,}376 & 1.60\% & 37/3{,}192 & 1.16\% & 118/4{,}528 & 2.61\% & 129/3{,}876 & 3.33\% & 67/5{,}404 & 1.24\% & 108/5{,}196 & 2.08\% & 2.05\% \\
Qwen-32B     & 2/4{,}404 & 0.05\% & 4/3{,}272 & 0.12\% & 5/3{,}376 & 0.15\% & 0/3{,}192 & 0.00\% & 18/4{,}528 & 0.40\% & 21/3{,}876 & 0.54\% & 16/5{,}404 & 0.30\% & 9/5{,}196 & 0.17\% & 0.23\% \\
\midrule
\multicolumn{17}{l}{All Models: 764\,/\,99{,}744} & 0.77\% \\
\bottomrule
\end{tabular}
\end{table*}

\section{Generality: Second Ensemble and Additional Benchmarks}
\label{app:additional_generality}

The main results use one primary agent ensemble.
This appendix tests whether the calibrated refusal mechanism transfers to a different agent composition and to different task formats.
In both cases the pipeline, non-conformity score, and target $\alpha$ are unchanged; each dataset receives its \emph{own} calibration split and round-specific threshold.
``The same procedure'' therefore means that no benchmark-specific heuristic tuning is introduced, not that a threshold estimated on MMLU-Pro is transferred elsewhere.

\subsection{Second Ensemble: Haiku + Nova Pro + DeepSeek-R1}
\label{app:secondary_ensemble}

We replicate the full pipeline with a second heterogeneous ensemble (Claude Haiku~4.5, Amazon Nova Pro, and DeepSeek-R1) on five MMLU-Pro domains (Engineering, Law, Chemistry, Physics, Mathematics).
This experiment runs $T{=}5$ debate rounds (indexed 0--4), whereas the primary experiment runs four.

At $\alpha{=}0.05$, coverage spans 90.5--96.3\% across domain--round cells (Table~\ref{tab:conformal_secondary}).
Most cells sit close to the 95\% target, but the transfer is not uniform: \emph{Physics is the clearest undercoverage case}, with coverage of 90.5--92.6\% against the 95\% target, a shortfall of 2.4--4.5 percentage points.
Set sizes remain difficulty-adaptive: Math converges to $|\calC|{=}1.00$ by round~4 (99.6\% singleton rate), while Law remains at 6.05 (13.6\% singleton rate), the same ordering as the primary ensemble.

Conformal stopping raises accuracy among acted cases on four of five domains, reaching 84.0--95.8\% singleton accuracy and up to 17.4 percentage points above consensus stopping (Engineering: 95.3\% vs.\ 77.8\%; Table~\ref{tab:early_stopping_secondary}); Math decreases by 1.2pp, mirroring the primary ensemble's $-$0.3pp on that domain.
On the safety side, conformal prediction intercepts 77.4\% of the 424 final-round wrong-consensus cases across 2{,}928 inference samples (Table~\ref{tab:combined_stopping_secondary}), compared with 81.9\% for the primary ensemble.
The qualitative pattern is preserved: interception is highest where consensus is least reliable (Law 94.6\%, Engineering 92.6\%) and lowest on Math (16.2\%), where agents are rarely wrong and rarely uncertain.
Taken together, these results support transfer across model composition while also exposing meaningful ensemble- and domain-dependent variability, particularly the Physics undercoverage.

\begin{table*}[t]
\centering
\scriptsize
\setlength{\tabcolsep}{3.5pt}
\begin{tabular}{@{}cl cc cc cc cc cc@{}}
\toprule
& & \multicolumn{2}{c}{\textbf{Engineering}} & \multicolumn{2}{c}{\textbf{Law}} & \multicolumn{2}{c}{\textbf{Chemistry}} & \multicolumn{2}{c}{\textbf{Physics}} & \multicolumn{2}{c}{\textbf{Math}} \\
\cmidrule(lr){3-4} \cmidrule(lr){5-6} \cmidrule(lr){7-8} \cmidrule(lr){9-10} \cmidrule(lr){11-12}
\textbf{Metric} & \textbf{Rd} & \textbf{.05} & \textbf{.10} & \textbf{.05} & \textbf{.10} & \textbf{.05} & \textbf{.10} & \textbf{.05} & \textbf{.10} & \textbf{.05} & \textbf{.10} \\
\midrule
\multirow{5}{*}{$\qhat_t$}
  & 0 & .983 & .960 & .963 & .941 & .963 & .772 & .940 & .734 & .750 & .644 \\
  & 1 & .987 & .950 & .980 & .960 & .984 & .722 & .901 & .587 & .652 & .347 \\
  & 2 & .993 & .972 & .990 & .970 & .988 & .732 & .934 & .483 & .684 & .184 \\
  & 3 & .997 & .984 & .990 & .979 & .992 & .717 & .950 & .394 & .702 & .089 \\
  & 4 & .997 & .983 & .990 & .980 & .993 & .723 & .950 & .318 & .725 & .065 \\
\midrule
\multirow{5}{*}{\shortstack[l]{Fixed-rd.\\Cov. (\%)}}
  & 0 & 95.7 & 90.1 & 94.9 & 89.3 & 95.8 & 88.7 & 92.6 & 87.5 & 93.9 & 87.6 \\
  & 1 & 95.7 & 88.9 & 95.5 & 89.1 & 96.3 & 88.0 & 91.5 & 87.1 & 93.8 & 90.1 \\
  & 2 & 95.3 & 90.5 & 94.9 & 90.0 & 95.4 & 88.0 & 90.9 & 86.6 & 93.5 & 90.8 \\
  & 3 & 95.3 & 90.5 & 94.4 & 90.0 & 95.9 & 88.0 & 90.8 & 86.6 & 93.5 & 91.0 \\
  & 4 & 95.3 & 88.0 & 94.7 & 90.6 & 96.1 & 88.0 & 90.5 & 86.6 & 93.5 & 91.3 \\
\midrule
\multirow{5}{*}{$|\calC_t|$}
  & 0 & 5.82 & 3.64 & 5.67 & 3.68 & 2.36 & 1.31 & 1.70 & 1.25 & 1.25 & 0.95 \\
  & 1 & 4.20 & 2.01 & 5.51 & 3.52 & 2.10 & 1.05 & 1.17 & 1.00 & 1.03 & 0.97 \\
  & 2 & 4.22 & 2.25 & 6.03 & 3.51 & 1.77 & 1.01 & 1.13 & 0.99 & 1.01 & 0.98 \\
  & 3 & 4.38 & 2.51 & 5.87 & 3.73 & 1.92 & 1.01 & 1.13 & 0.99 & 1.01 & 0.98 \\
  & 4 & 3.98 & 2.17 & 6.05 & 3.78 & 1.92 & 1.01 & 1.12 & 0.99 & 1.00 & 0.98 \\
\midrule
\multirow{5}{*}{\shortstack[l]{Adaptive\\Sing. Rate}}
  & 0 & 20.6 & 30.9 & 3.8 & 8.0 & 52.3 & 72.4 & 58.8 & 75.7 & 77.5 & 92.9 \\
  & 1 & 40.8 & 60.2 & 7.6 & 14.3 & 69.8 & 95.4 & 86.2 & 99.1 & 97.2 & 96.7 \\
  & 2 & 43.9 & 64.3 & 10.0 & 19.6 & 75.4 & 98.9 & 89.7 & 99.4 & 99.3 & 97.8 \\
  & 3 & 45.4 & 65.6 & 12.3 & 21.6 & 77.9 & 99.3 & 90.8 & 99.4 & 99.4 & 98.1 \\
  & 4 & 47.8 & 68.7 & 13.6 & 23.6 & 79.2 & 99.5 & 91.2 & 99.4 & 99.6 & 98.4 \\
\midrule
\multirow{5}{*}{\shortstack[l]{Adaptive\\Sing. Acc.}}
  & 0 & 96.0 & 94.7 & 90.5 & 90.9 & 99.0 & 91.0 & 96.1 & 89.2 & 94.8 & 93.3 \\
  & 1 & 96.9 & 92.2 & 85.7 & 88.6 & 91.9 & 80.8 & 88.8 & 81.6 & 91.7 & 88.5 \\
  & 2 & 86.7 & 80.0 & 76.9 & 72.4 & 84.4 & 70.0 & 65.2 & 0.0 & 57.1 & 71.4 \\
  & 3 & 71.4 & 66.7 & 76.9 & 81.8 & 78.6 & 100.0 & 71.4 & --- & 100.0 & 50.0 \\
  & 4 & 100.0 & 53.3 & 85.7 & 72.7 & 100.0 & 100.0 & 66.7 & --- & 100.0 & 100.0 \\
\bottomrule
\end{tabular}
\caption{Conformal prediction results for the secondary ensemble (Haiku + Nova Pro + DeepSeek-R1) with early stopping on five MMLU-Pro domains. Metric names match Table~\ref{tab:conformal}; rows~0--4 correspond to debate rounds ($T{=}5$). Set sizes and singleton rates follow the same difficulty-adaptive pattern as the primary ensemble, but Physics undercovers at $\alpha{=}0.05$ (90.5--92.6\% against a 95\% target). As in Table~\ref{tab:conformal}, \textbf{Singleton Acc.}\ is measured among the samples \emph{newly} resolved at that round, so late-round cells rest on very small denominators and are noisy; dashes mark rounds where no new sample was resolved.}
\label{tab:conformal_secondary}
\end{table*}

\begin{table}[t]
\centering
\resizebox{\columnwidth}{!}{%
\scriptsize
\setlength{\tabcolsep}{2pt}
\begin{tabular}{@{}l r cc cc cc r@{}}
\toprule
& & \multicolumn{2}{c}{\textbf{Avg Round}} & \multicolumn{2}{c}{\textbf{Singleton \%}} & \multicolumn{2}{c}{\textbf{Singleton Acc.}} & \\
\cmidrule(lr){3-4} \cmidrule(lr){5-6} \cmidrule(lr){7-8}
\textbf{Domain} & $n$ & \textbf{Cons.} & \textbf{Conf.} & \textbf{Cons.} & \textbf{Conf.} & \textbf{Cons.} & \textbf{Conf.} & $\Delta$\textbf{Acc} \\
\midrule
Engineering & 485 & 1.91 & 2.49 & 98.1 & 47.8 & 77.8 & 95.3 & +17.4 \\
Law         & 551 & 1.97 & 3.66 & 97.8 & 13.6 & 69.3 & 84.0 & +14.7 \\
Chemistry   & 566 & 1.51 & 1.25 & 99.5 & 79.2 & 89.6 & 95.8 & +6.1 \\
Physics     & 650 & 1.52 & 0.75 & 99.4 & 91.2 & 89.2 & 92.2 & +3.0 \\
Math        & 676 & 1.43 & 0.27 & 100.0 & 99.6 & 94.6 & 93.5 & $-$1.2 \\
\bottomrule
\end{tabular}}
\caption{Consensus vs.\ conformal early stopping for the secondary ensemble ($\alpha{=}0.05$, final round). Conformal stopping yields higher singleton accuracy on 4 of 5 domains, with the same pattern as the primary ensemble (Table~\ref{tab:early_stopping}): largest gains on harder domains, no advantage on Math. \emph{Note that the conventions differ from Table~\ref{tab:early_stopping}}: here \textbf{Avg Round} is a 0-indexed mean over \emph{all} samples (unresolved samples counted as $T{-}1{=}4$), and \textbf{Singleton Acc.}\ includes all samples resolved through round~4 (whereas Table~\ref{tab:early_stopping} averages 1-indexed rounds over resolved samples only and excludes final-round stops from Singleton Acc.). Domains where most samples resolve immediately fall below~1 in Avg Round under this convention.}
\label{tab:early_stopping_secondary}
\end{table}

\begin{table}[t]
\centering
\resizebox{\columnwidth}{!}{%
\scriptsize
\setlength{\tabcolsep}{2pt}
\begin{tabular}{@{}l r cc rr rr@{}}
\toprule
& & \multicolumn{2}{c}{\textbf{Singleton \%}} & \multicolumn{2}{c}{\textbf{Wrong Cons.}} & \multicolumn{2}{c}{\textbf{Correct Cons.}} \\
\cmidrule(lr){3-4} \cmidrule(lr){5-6} \cmidrule(lr){7-8}
\textbf{Domain} & $n$ & \textbf{Cons.} & \textbf{Conf.} & \textbf{Count} & \textbf{Rej.\%} & \textbf{Count} & \textbf{Rej.\%} \\
\midrule
Engineering & 485 & 98.1 & 47.8 &  94 & 92.6 & 362 & 54.1 \\
Law         & 551 & 97.8 & 13.6 & 167 & 94.6 & 360 & 86.1 \\
Chemistry   & 566 & 99.5 & 79.2 &  53 & 69.8 & 498 & 23.5 \\
Physics     & 650 & 99.4 & 91.2 &  73 & 54.8 & 564 &  5.9 \\
Math        & 676 & 100.0 & 99.6 &  37 & 16.2 & 637 &  0.5 \\
\midrule
\textbf{All} & \textbf{2{,}928} & & & \textbf{424} & \textbf{77.4} & \textbf{2{,}421} & \textbf{27.2} \\
\bottomrule
\end{tabular}}
\caption{Wrong-consensus final-round diagnostic for the secondary ensemble ($\alpha{=}0.05$, round~4). \textbf{Rej.\%}~=~fraction of final-round wrong-consensus cases for which the conformal set is non-singleton. The 77.4\% aggregate rate shows a similar final-round interception pattern under the second agent ensemble, with a modest decrease from 81.9\% in the primary ensemble (Table~\ref{tab:combined_stopping}).}
\label{tab:combined_stopping_secondary}
\end{table}

\subsection{Additional Benchmarks: LogiQA and GPQA Extended}
\label{app:additional_benchmarks}

We additionally evaluate the primary ensemble (Claude Haiku, DeepSeek-R1, Qwen-3~32B) on two benchmarks outside MMLU-Pro: \textbf{LogiQA} \citep{liu2020logiqa} ($n{=}8{,}678$; logical reasoning) and \textbf{GPQA Extended} \citep{rein2023gpqa} ($n{=}546$; graduate-level science).
Both use \emph{four} answer options, whereas MMLU-Pro uses ten, and both are run with the same four-round debate procedure and a separate random 50/50 calibration--test split (LogiQA: 4{,}339 calibration / 4{,}339 test; GPQA Extended: 273 / 273).
This experiment therefore tests whether the procedure transfers across task formats and option-space sizes without benchmark-specific threshold tuning.

Per-round thresholds, coverage, set sizes, and stopping outcomes for both benchmarks are reported alongside MMLU-Pro in main-text Tables~\ref{tab:conformal}--\ref{tab:combined_stopping} and are not repeated here.
Two observations specific to the smaller label space are worth recording.
First, average set sizes remain well above~1 (3.00 on LogiQA and 2.75 on GPQA Extended at round~3, out of only four options), so both benchmarks retain substantial residual uncertainty and rarely reach singletons; the framework does not degenerate to near-universal automation simply because the option space shrank.
Second, wrong-consensus interception is \emph{higher} than on MMLU-Pro (90.8\% and 87.3\% vs.\ 81.9\%) but so is correct-consensus over-rejection (74.2\% and 61.6\% vs.\ 31.9\%), so the safety gain on these benchmarks is bought with a proportionally larger loss of automation rather than with better selectivity.

\section{Domain-Level Analysis}
\label{app:domain_analysis}

\paragraph{Why do some domains resist singleton convergence?}
The prediction set sizes reveal a clear pattern tied to domain reasoning type.
\emph{Formal-reasoning} domains (Math, Physics, Chemistry) converge rapidly to singletons: Math reaches 98.2\% singleton rate at $\alpha{=}0.05$ by round~3, consistent with their higher accuracy and more concentrated social probability distributions.
\emph{Interpretive-reasoning} domains (Law, Engineering, Health) maintain large sets: Law's average set size of 6.91 at round~3 is consistent with greater residual model uncertainty or lower score concentration on these tasks, where multiple answers may remain plausible under different legal interpretations.
This domain-adaptive behavior emerges entirely from the calibration procedure, without any per-domain tuning.

\paragraph{Trade-off: reliability vs.\ automation.}
The gains in singleton accuracy come at the cost of automation coverage.
On Math, where consensus is already reliable (94.9\%), conformal stopping offers no accuracy advantage ($-$0.3pp) while resolving a comparable fraction of cases (98.2\% vs.\ 97.8\%).
Conversely, on Law, conformal stopping improves accuracy by 22.1pp but resolves only 6.2\% of samples; the remaining 93.8\% are escalated to human review.
This is not a weakness but an intended feature: on domains with lower score concentration, the method escalates most cases rather than committing to automated action.
The operating point is user-adjustable: increasing $\alpha$ from 0.05 to 0.10 raises Law's singleton rate from 6.2\% to 20.9\% (Table~\ref{tab:conformal}), trading a share of the interception rate for roughly three times the automation.

\end{document}